\documentclass{article}

\usepackage{amsmath,amsbsy,amsfonts,amssymb,amsthm,color,dsfont,mleftright,nccmath}
\usepackage{url}
\usepackage{dsfont} 

\usepackage[disable]{todonotes}

\newcommand{\hide}[1]{}

\newtheorem{lemma}{Lemma}
\newtheorem{claim}{Claim}
\newtheorem{proposition}{Proposition}
\newtheorem{theorem}{Theorem}
\newtheorem{corollary}{Corollary}

\usepackage{latexsym}
\usepackage{mathtools}
\usepackage{enumerate}
\usepackage[ruled]{algorithm}
\usepackage{algorithmic}
\usepackage{hyperref}
\hypersetup{
    bookmarks=true,         
    unicode=false,          
    pdftoolbar=true,        
    pdfmenubar=true,        
    pdffitwindow=false,     
    pdfstartview={FitH},    
    pdftitle={My title},    
    pdfauthor={Author},     
    pdfsubject={Subject},   
    pdfcreator={Creator},   
    pdfproducer={Producer}, 
    pdfkeywords={keyword1} {key2} {key3}, 
    pdfnewwindow=true,      
    colorlinks=true,       
    linkcolor=red,          
    citecolor=blue,        
    filecolor=magenta,      
    urlcolor=black          
}
\usepackage[numbers,sort&compress]{natbib}
\usepackage{enumitem}
\usepackage{bibunits}

\usepackage[capitalize]{cleveref}
\usepackage{fullpage}

\def\ddefloop#1{\ifx\ddefloop#1\else\ddef{#1}\expandafter\ddefloop\fi}

\def\ddef#1{\expandafter\def\csname b#1\endcsname{\ensuremath{\mathbf{#1}}}}
\ddefloop ABCDEFGHIJKLMNOPQRSTUVWXYZ\ddefloop

\def\ddef#1{\expandafter\def\csname bb#1\endcsname{\ensuremath{\mathbb{#1}}}}
\ddefloop ABCDEFGHIJKLMNOPQRSTUVWXYZ\ddefloop

\def\ddef#1{\expandafter\def\csname c#1\endcsname{\ensuremath{\mathcal{#1}}}}
\ddefloop ABCDEFGHIJKLMNOPQRSTUVWXYZ\ddefloop

\def\ddef#1{\expandafter\def\csname v#1\endcsname{\ensuremath{\boldsymbol{#1}}}}
\ddefloop ABCDEFGHIJKLMNOPQRSTUVWXYZabcdefghijklmnopqrstuvwxyz\ddefloop

\def\ddef#1{\expandafter\def\csname v#1\endcsname{\ensuremath{\boldsymbol{\csname #1\endcsname}}}}
\ddefloop {alpha}{beta}{gamma}{delta}{epsilon}{varepsilon}{zeta}{eta}{theta}{vartheta}{iota}{kappa}{lambda}{mu}{nu}{xi}{pi}{varpi}{rho}{varrho}{sigma}{varsigma}{tau}{upsilon}{phi}{varphi}{chi}{psi}{omega}{Gamma}{Delta}{Theta}{Lambda}{Xi}{Pi}{Sigma}{varSigma}{Upsilon}{Phi}{Psi}{Omega}\ddefloop

\newcommand\veps{\ensuremath{\varepsilon}}
\newcommand\eps{\ensuremath{\epsilon}}

\renewcommand\t{{\ensuremath{\scriptscriptstyle{\top}}}}

\DeclareMathOperator{\Diag}{Diag}

\DeclareMathOperator{\var}{var}
\DeclareMathOperator{\Var}{Var}

\newcommand\wt{\ensuremath{\widetilde}}
\newcommand\wh{\ensuremath{\widehat}}
\renewcommand\v{\ensuremath{\boldsymbol}}

\newcommand\parens[1]{(#1)}
\newcommand\norm[1]{\|#1\|}
\newcommand\braces[1]{\{#1\}}
\newcommand\brackets[1]{[#1]}
\newcommand\ceil[1]{\lceil#1\rceil}
\newcommand\abs[1]{|#1|}

\newcommand\dotp[1]{\langle #1 \rangle}

\newcommand\Parens[1]{\left(#1\right)}
\newcommand\Norm[1]{\left\|#1\right\|}
\newcommand\Braces[1]{\left\{#1\right\}}
\newcommand\Brackets[1]{\left[#1\right]}
\newcommand\Ceil[1]{\left\lceil#1\right\rceil}
\newcommand\Floor[1]{\left\lfloor#1\right\rfloor}
\newcommand\Abs[1]{\left|#1\right|}
\newcommand\Ind[1]{\ensuremath{\mathds{1}\left\{#1\right\}}}

\newcommand\pimin{\ensuremath{\pi_{\star}}}
\newcommand\gap{\ensuremath{\gamma_{\star}}}
\newcommand\slem{\ensuremath{\lambda_{\star}}}

\newcommand\hatpimin{\ensuremath{\hat\pi_{\star}}}
\newcommand\hatgap{\ensuremath{\hat\gamma_{\star}}}

\newcommand\Geom{\ensuremath{\operatorname{Geom}}}
\newcommand\Sym{\ensuremath{\operatorname{Sym}}}
\newcommand\errm{\ensuremath{\v{\cE}_{\vM}}}
\newcommand\errp{\ensuremath{\v{\cE}_{\vpi}}}
\newcommand\errP{\ensuremath{\v{\cE}_{\vP}}}
\newcommand\errpil{\ensuremath{\v{\cE}_{\vpi,1}}}
\newcommand\errpir{\ensuremath{\v{\cE}_{\vpi,2}}}
\newcommand\tv{\ensuremath{\operatorname{tv}}}
\newcommand\tmix{\ensuremath{t_{\operatorname{mix}}}}
\newcommand\trelax{\ensuremath{t_{\operatorname{relax}}}}

\newcommand{\iset}[1]{[#1]}
\newcommand{\opencloseint}[1]{\ensuremath{(#1]}}

\newcommand{\Dvpi}{\Diag(\vpi)}
\newcommand{\Dvpit}{\Diag(\vpi^{(t)})}

\newcommand{\EE}[1]{\bbE\left(#1\right)}

\newcommand\giAh{\ensuremath{\widehat{\boldsymbol{A}}^{\raisebox{-2.9pt}{$\scriptstyle\#$}}}}
\newcommand\giA{\ensuremath{\boldsymbol{A}^{\raisebox{-0pt}{$\scriptstyle\#$}}}}
\newcommand{\defeq}{:=}

\newcommand\emptail{\ensuremath{\tau_{n,\delta}}}

\newcommand{\st}{\ensuremath{\mathrm{s.t.}}}

\renewcommand\citet\cite
\renewcommand\citep\cite

\title{%
  Mixing Time Estimation in Reversible Markov Chains from a Single
  Sample Path%
}

\author{%
  Daniel Hsu \\
  Columbia University \\
  \url{djhsu@cs.columbia.edu}
  \and
  Aryeh Kontorovich \\
  Ben-Gurion University \\
  \url{karyeh@cs.bgu.ac.il}
  \and
  Csaba Szepesv\'ari \\
  University of Alberta \\
  \url{szepesva@cs.ualberta.ca}
}

\begin{document}
\maketitle

\begin{abstract} 
This article provides the first procedure for computing a fully
data-dependent interval that traps the mixing time
$t_{\text{mix}}$ of a finite
reversible ergodic Markov chain at a prescribed confidence level.  The
interval is computed from a single finite-length sample path from the
Markov chain, and does not require the knowledge of any parameters of
the chain.  This stands in contrast to previous approaches, which
either only provide point estimates, or require a reset mechanism, or
additional prior knowledge.
The interval is constructed around the relaxation time
$t_{\text{relax}}$, which is strongly related to the mixing time, and
the width of the interval converges to zero roughly
at a $\sqrt{n}$ rate, where $n$ is the length of the sample path.
Upper and lower bounds are given on the number of samples required to
achieve constant-factor multiplicative accuracy.  The lower bounds
indicate that, unless further restrictions are placed on the chain, no
procedure can achieve this accuracy level before seeing each state at
least $\Omega(t_{\text{relax}})$ times on the average.  Finally, future
directions of research are identified.

\end{abstract} 

\section{Introduction}\label{sec:intro}
This work
tackles the challenge of constructing
fully empirical bounds on the
mixing time of Markov chains based on a single sample path.
Let $(X_t)_{t=1,2,\dotsc}$ be an irreducible, aperiodic
time-homogeneous Markov chain on a finite state space $[d] :=
\{1,2,\dotsc,d\}$ with transition matrix $\vP$.
Under this assumption, the chain converges to its unique stationary
distribution $\vpi =
(\pi_i)_{i=1}^d$ regardless of the initial state distribution $\vq$: 
\[
  \lim_{t\to\infty} {\Pr}_{\vq}\Parens{X_t = i}
  = \lim_{t\to\infty} (\vq \vP^t)_i = \pi_i
  \quad \text{for each $i \in [d]$} .
\]
The \emph{mixing time} $\tmix$ of the Markov chain is the
number of time steps required
for the chain to
be within a fixed threshold of
its stationary
distribution:
\begin{align}
\label{eq:mixtimedef}
  \tmix
  :=
  \min\Braces{
    t \in \bbN :
    \sup_{\vq}
    \max_{A\subset [d]}
    \Abs{
      \textstyle\Pr_{\vq}\Parens{ X_t \in A } - \vpi(A)
    }
    \leq 1/4
  }\,.
\end{align}
Here, $\vpi(A) = \sum_{i\in A} \pi_i$ is the probability assigned to
set $A$ by $\vpi$, and the supremum is over all possible initial
distributions $\vq$.
The problem studied in this work is the construction of a non-trivial
confidence interval $C_n = C_n(X_1,X_2,\dotsc,X_n,\delta) \subset
[0,\infty]$, based only on the observed sample path
$(X_1,X_2,\dotsc,X_n)$ and $\delta \in (0,1)$, that
succeeds with probability $1-\delta$
in trapping
the value of the
mixing time $\tmix$.

This problem is motivated by the numerous scientific applications and
machine learning tasks in which
the quantity of interest is
the
mean $\vpi(f) = \sum_i \pi_i f(i)$
for some function $f$ of the states of
a Markov chain.
This is the setting of the celebrated Markov Chain Monte Carlo (MCMC)
paradigm~\cite{liu2001monte}, but the problem also arises in
performance prediction involving time-correlated data, as is common in
reinforcement learning~\cite{sutton98}.
Observable bounds on mixing times are useful in the design and
diagnostics of these methods; they yield effective approaches to
assessing the estimation quality, even when \emph{a priori} knowledge
of the mixing time or correlation structure is unavailable.

\paragraph{Main results.}
We develop the first procedure for constructing non-trivial
and fully empirical confidence intervals for Markov mixing time.
Consider a reversible ergodic Markov chain on $d$ states with absolute
spectral gap $\gap$ and stationary distribution minorized by $\pimin$. 
As is well-known \citep[Theorems~12.3 and~12.4]{LePeWi08},
\begin{equation}
  \label{eq:mixing-time-bound}
  \Parens{\trelax - 1} \ln2
  \ \leq \ \tmix
  \ \leq \ \trelax \ln \mfrac4{\pimin}
\end{equation}
where $\trelax := 1/\gap$ is the \emph{relaxation time}.
Hence, it suffices to estimate $\gap$ and $\pimin$. 
Our main results are summarized as follows.
\begin{enumerate}
  \item
    In \cref{sec:rates-lower}, we show that in some problems 
    $n = \Omega\parens{ (d\log d)/\gap + 1/\pimin}$ observations are necessary for any procedure
    to guarantee constant multiplicative accuracy in estimating $\gap$
    (\cref{thm:lb-pimin,thm:lb-gap}).
    Essentially, in some problems \emph{every} state may need to be visited about
    $\log(d)/\gap$ times, on average, before an accurate estimate of
    the mixing time can be provided, regardless of the actual estimation
    procedure used.

  \item
    In \cref{sec:rates-upper}, we give a point-estimator for $\gap$,
    and prove in \cref{thm:err} that it achieves multiplicative
    accuracy from \emph{a single sample path} of length
    $\tilde{O}\parens{1/(\pimin\gap^3)}$.\footnote{The
    $\tilde{O}(\cdot)$ notation suppresses logarithmic factors.}
    We also provide a point-estimator for $\pimin$ that requires a sample
    path of length $\tilde{O}\parens{1/(\pimin\gap)}$.
    This establishes the feasibility of \emph{estimating} the mixing
    time in this setting.
    However, the valid confidence intervals suggested by
    \cref{thm:err} depend on the unknown quantities $\pimin$ and
    $\gap$.
    We also discuss the importance of reversibility, and some possible
    extensions to non-reversible chains.

  \item
    In \cref{sec:empirical}, the construction of
    valid \emph{fully empirical confidence intervals} for $\pimin$ and
    $\gap$ are considered. First, the difficulty of the task is explained,
    i.e., why the standard approach of turning the finite time confidence intervals of
    \cref{thm:err} into a fully empirical one fails.
	Combining several results from perturbation theory in a novel fashion
	we propose a new procedure and prove that it avoids slow convergence
    (\cref{thm:empirical}).
    We also explain how to combine the empirical confidence intervals
    from \cref{alg:empest} with the non-empirical bounds from
    \cref{thm:err} to produce valid empirical confidence intervals.
    We prove in \cref{thm:combined} that the width of these new
    intervals converge to zero asymptotically at least as fast as
    those from either \cref{thm:err} and \cref{thm:empirical}.

\end{enumerate}

\paragraph{Related work.}
There is a vast statistical literature on estimation in Markov chains.
For instance, it is known that under the assumptions on
$(X_t)_t$ from above, the law of large numbers guarantees that
the sample mean $\vpi_n(f) \defeq \frac1n \sum_{t=1}^n f(X_t)$
converges almost surely to $\vpi(f)$~\cite{meyn1993markov}, while the
central limit theorem tells us that as $n\to \infty$, the distribution
of the deviation $\sqrt{n}( \vpi_n(f)-\vpi(f))$ will be normal with
mean zero and asymptotic variance $\lim_{n\to\infty}
n\Var\Parens{\vpi_n(f)}$~\cite{kipnis1986central}.

Although these asymptotic results help us understand the limiting
behavior of the sample mean over a Markov chain, they say little about
the finite-time non-asymptotic behavior, which is often needed for the
prudent evaluation of a method or even its algorithmic design
\cite{
MCMCDiscussion93%
,DBLP:conf/valuetools/KontoyiannisLM06%
,BBL06%
,MniSzeAu08%
,MauPo09%
,LiLiWaSt11:KWIK%
,flegal2011implementing%
,Gyori-paulin15%
,SwaJoa15:LoggedBandit%
}.
To address this need, numerous works have developed Chernoff-type
bounds on $\Pr\parens{ |\vpi_n(f)-\vpi(f)| > \eps }$, thus providing
valuable tools for non-asymptotic probabilistic
analysis~\cite{gillman1998chernoff,leon2004optimal,DBLP:conf/valuetools/KontoyiannisLM06,
paulin15}.
These probability bounds are larger than corresponding bounds for
independent and identically distributed (iid) data due to the temporal
dependence; intuitively, for the Markov chain to yield a fresh draw
$X_{t'}$ that behaves as if it was independent of $X_t$, one must wait
$\Theta(\tmix)$ time steps.
Note that the bounds generally depend on distribution-specific
properties of the Markov chain (e.g., $\vP$, $\tmix$, $\gap$), which are often
unknown \emph{a priori} in practice.
Consequently, much effort has been put towards estimating these
unknown quantities, especially in the context of MCMC diagnostics, in
order to provide data-dependent assessments of estimation
accuracy~\cite[e.g.,][]{MCMCDiscussion93,GaSmi00:eigval,jones2001,flegal2011implementing,1209.0703,Gyori-paulin15}.
However, these approaches generally only provide asymptotic
guarantees, and hence fall short of our goal of empirical bounds that
are valid with any finite-length sample path.

Learning with dependent data is another main motivation to our work.
Many results from statistical learning and empirical process theory
have been extended to sufficiently fast mixing, dependent
data~\citep[e.g.,][]{Yu94,MR1921877,gamarnik03,MoRo08,MoRo09,DBLP:conf/nips/SteinwartC09,Steinwart2009175},
providing learnability assurances (e.g., generalization error bounds).
These results are often given in terms of mixing coefficients, which
can be consistently estimated in some cases~\citet{McDoShaSche11}.
However, the convergence rates of the estimates
from~\citet{McDoShaSche11}, which are needed to derive confidence
bounds, are given in terms of unknown mixing coefficients.
When the data comes from a Markov chain, these mixing coefficients can
often be bounded in terms of mixing times, and hence our main results
provide a way to make them fully empirical, at least in the limited setting we study.

It is possible to eliminate many of the difficulties presented above
when allowed more flexible access
to the Markov chain.
For example, given a sampling oracle that generates
independent transitions from any given state (akin to a ``reset''
device), the mixing time becomes an efficiently testable property in
the sense studied in~\citet{BaFoRuSmiWhi00,BaFoRuSmiWhi13}.
On the other hand, when one only has a circuit-based description of
the transition probabilities of a Markov chain over an
exponentially-large state space, there are complexity-theoretic
barriers for many MCMC diagnostic problems~\citet{BhaBoMo11}.

\section{Preliminaries}\label{sec:prelim}
\subsection{Notations}
\label{sec:notation}

We denote the set of positive integers by $\bbN$, and
the set of the first $d$ positive integers $\{1,2,\dotsc,d\}$ by $\iset{d}$.
The non-negative part of a real number $x$ is $[x]_+ := \max\{0,x\}$,
and $\ceil{x}_+ := \max\{0,\ceil{x}\}$.
We use $\ln(\cdot)$ for natural logarithm, and $\log(\cdot)$ for
logarithm with an arbitrary constant base.
Boldface symbols are used for vectors and matrices (e.g., $\vv$,
$\vM$), and their entries are referenced by subindexing (e.g., $v_i$,
$M_{i,j}$).
For a vector $\vv$, $\norm{\vv}$ denotes its Euclidean norm; for a
matrix $\vM$, $\norm{\vM}$ denotes its spectral norm.
We use $\Diag(\vv)$ to denote the diagonal matrix whose $(i,i)$-th
entry is $v_i$.
The probability simplex is denoted by $\Delta^{d-1} = \{ \vp
\in [0,1]^d : \sum_{i=1}^d p_i = 1 \}$, and we regard vectors in
$\Delta^{d-1}$ as row vectors.

\subsection{Setting}
\label{sec:setting}

Let $\vP \in (\Delta^{d-1})^d \subset [0,1]^{d \times d}$ be a $d
\times d$ row-stochastic matrix for an ergodic (i.e., irreducible and
aperiodic) Markov chain.
This implies there is a unique stationary distribution $\vpi \in
\Delta^{d-1}$ with $\pi_i > 0$ for all $i \in [d]$~\citep[Corollary
1.17]{LePeWi08}.
We also assume that $\vP$ is \emph{reversible} (with respect to
$\vpi$):
\begin{align}
\label{eq:reversibility}
  \pi_i P_{i,j} = \pi_j P_{j,i} ,
  \quad i,j \in [d] .
\end{align}
The minimum stationary probability is denoted by $\pimin := \min_{i
\in [d]} \pi_i$.

Define the matrices
\begin{align*}
\vM := \Diag(\vpi) \vP \quad \text{and} \quad
\vL := \Diag(\vpi)^{-1/2} \vM \Diag(\vpi)^{-1/2}\,.
\end{align*}
The $(i,j)$th entry of the matrix $M_{i,j}$ contains the \emph{doublet
probabilities} associated with $\vP$: $M_{i,j} = \pi_i P_{i,j}$ is the
probability of seeing state $i$ followed by state $j$ when the chain
is started from its stationary distribution.
The matrix $\vM$ is symmetric on account of the reversibility of
$\vP$, and hence it follows that $\vL$ is also symmetric.
(We will strongly exploit the symmetry in our results.)
Further, $\vL = \Diag(\vpi)^{1/2} \vP \Diag(\vpi)^{-1/2}$, hence $\vL$
and $\vP$ are similar and thus their eigenvalue systems are identical.
Ergodicity and reversibility imply that the eigenvalues of $\vL$ are
contained in the interval $\opencloseint{-1,1}$, and that $1$ is an
eigenvalue of $\vL$ with multiplicity $1$~\citep[Lemmas 12.1 and
12.2]{LePeWi08}.
Denote and order the eigenvalues of $\vL$ as
\[
  1 = \lambda_1 > \lambda_2 \geq \dotsb \geq \lambda_d > -1 .
\]
Let $\slem := \max\{ \lambda_2,\, |\lambda_d| \}$, and define the
(absolute) spectral gap to be $\gap := 1-\slem$, which is strictly
positive on account of ergodicity.

Let $(X_t)_{t\in\bbN}$ be a Markov chain whose transition
probabilities are governed by $\vP$.
For each $t \in \bbN$, let $\vpi^{(t)} \in \Delta^{d-1}$ denote the
marginal distribution of $X_t$, so
\[
  \vpi^{(t+1)} = \vpi^{(t)} \vP ,
  \quad t \in \bbN .
\]
Note that the initial distribution $\vpi^{(1)}$ is arbitrary,
and need not be the stationary distribution $\vpi$.

The goal is to estimate $\pimin$ and $\gap$ from the length $n$ sample
path $(X_t)_{t \in [n]}$, and also to construct fully empirical
confidence intervals that $\pimin$ and $\gap$ with high probability;
in particular, the construction of the intervals should not depend on
any unobservable quantities, including $\pimin$ and $\gap$ themselves.
As mentioned in the introduction,
it is well-known that the \emph{mixing time} of the Markov chain
$\tmix$ (defined in Eq.~\ref{eq:mixtimedef})
is bounded in terms of $\pimin$ and $\gap$, as shown in
\cref{eq:mixing-time-bound}.
Moreover, convergence rates for empirical processes on Markov chain
sequences are also often given in terms of mixing coefficients that
can ultimately be bounded in terms of $\pimin$ and
$\gap$ (as we will show in the proof of our first result).
Therefore, valid confidence intervals for $\pimin$ and $\gap$ can be
used to make these rates fully observable.

\section{Point estimation}\label{sec:rates}
In this section, we present lower and upper bounds on achievable rates
for estimating the spectral gap as a function of the length of the
sample path $n$.

\subsection{Lower bounds}
\label{sec:rates-lower}
The purpose of this section is to show lower bounds on the number of observations
necessary to achieve a fixed multiplicative (or even just additive) accuracy in estimating the spectral gap $\gap$.
By \cref{eq:mixing-time-bound}, the multiplicative accuracy lower bound for $\gap$
gives the same lower bound for estimating the mixing time.
Our first result holds even for two state Markov chains and shows that a sequence length of $\Omega(1/\pimin)$
is necessary to achieve even a constant \emph{additive} accuracy in estimating $\gap$.
\begin{theorem}
  \label{thm:lb-pimin}
  Pick any $\bar\pi \in (0,1/4)$.
  Consider any estimator $\hatgap$ that takes as input a random sample
  path of length $n \leq 1/(4\bar\pi)$ from a Markov chain starting
  from any desired initial state distribution.
  There exists a two-state ergodic and reversible Markov chain
  distribution with spectral gap $\gap \geq 1/2$ and minimum
  stationary probability $\pimin \geq \bar\pi$ such that
  \[
    \Pr\Brackets{ |\hatgap - \gap| \geq 1/8 } \geq 3/8 .
  \]
\end{theorem}
Next, considering $d$ state chains, we show that 
a sequence of length $\Omega(d\log(d)/\gap)$ is required
to estimate $\gap$ up to a constant multiplicative accuracy.
Essentially, the sequence may have to visit all $d$ states at least
$\log(d)/\gap$ times each, on average.
This holds \emph{even} if $\pimin$ is within a factor of two of the
\emph{largest} possible value of $1/d$ that it can take, i.e., when
$\vpi$ is nearly uniform.
\begin{theorem}
  \label{thm:lb-gap}
  There is an absolute constant $c>0$ such that the following holds.
  Pick any positive integer $d \geq 3$ and any $\bar\gamma \in
  (0,1/2)$.
  Consider any estimator $\hatgap$ that takes as input a random sample
  path of length $n < c d\log(d) / \bar\gamma$ from a $d$-state
  reversible Markov chain starting from any desired initial state
  distribution.
  There is an ergodic and reversible Markov chain distribution
  with spectral gap $\gap \in [\bar\gamma,2\bar\gamma]$ and minimum
  stationary probability $\pimin \geq 1/(2d)$ such that
  \[
    \Pr\Brackets{ |\hatgap - \gap| \geq \bar\gamma/2} \geq 1/4 .
  \]
\end{theorem}

The proofs of \cref{thm:lb-pimin,thm:lb-gap} are given in
\cref{app:lower}. 

\subsection{A plug-in based point estimator and its accuracy}
\label{sec:rates-upper}
Let us now consider the problem of estimating $\gap$.
For this, we construct a natural plug-in estimator.
Along the way, we also provide an estimator for the minimum stationary probability,
allowing one to use the bounds from \cref{eq:mixing-time-bound} to trap
the mixing time.

Define the random matrix $\wh\vM \in [0,1]^{d \times d}$ and random
vector $\hat\vpi \in \Delta^{d-1}$ by
\begin{align*}
  \wh{M}_{i,j}
  & := \frac{|\{ t \in [n-1] : (X_t,X_{t+1}) = (i,j) \}|}{n-1}
  , \quad i,j \in [d]\,,
  \\
  \hat{\pi}_i
  & := \frac{|\{ t \in [n] : X_t = i \}|}{n}
  , \quad i \in [d]
  \,.
\end{align*}
Furthermore, define
\[
  \Sym(\wh\vL) := \frac12 \parens{ \wh\vL + \wh\vL^\t }
\]
to be the symmetrized version of the (possibly non-symmetric) matrix
\[
  \wh\vL := \Diag(\hat\vpi)^{-1/2} \wh\vM \Diag(\hat\vpi)^{-1/2}
  .
\]
Let $\hat\lambda_1 \geq \hat\lambda_2 \geq \dotsb \geq \hat\lambda_d$
be the eigenvalues of $\Sym(\wh\vL)$.
Our estimator of the minimum stationary probability $\pimin$ is
\[
  \hatpimin := \min_{i \in [d]} \hat\pi_i ,
\]
and our estimator of the spectral gap $\gap$ is
\[
  \hatgap := 1 - \max\{ \hat\lambda_2, |\hat\lambda_d| \} .
\]

These estimators have the following accuracy guarantees:
\begin{theorem}
  \label{thm:err}
  There exists an absolute constant $C>0$ such that the following
  holds.
  Assume the estimators $\hatpimin$ and $\hatgap$ described above are
  formed from a sample path of length $n$ from an ergodic and
  reversible Markov chain.
  Let $\gap>0$ denote the spectral gap and $\pimin>0$ the minimum
  stationary probability.
  For any $\delta \in (0,1)$, with probability at least $1-\delta$,
  \begin{equation}\label{eq:piminbound}
    \Abs{\hatpimin-\pimin}
    \le
    C \,
    \Parens{
      \sqrt{\frac{\pimin\log\frac{d}{\pimin\delta}}{\gap n}}
      +
      \frac{\log\frac{d}{\pimin\delta}}{\gap n}
    }
  \end{equation}
  and
  \begin{equation}\label{eq:gapbound}
    \Abs{\hatgap-\gap}
    \leq
    C \,
    \Parens{
      \sqrt{\frac{\log\frac{d}{\delta}\cdot\log\frac{n}{\pimin\delta}}{\pimin\gap n}}
      + \frac{\log\frac{1}{\gap}}{\gap n}  
    }
    \,.
  \end{equation}
\end{theorem}

\Cref{thm:err} implies that the sequence lengths required to estimate
$\pimin$ and
$\gap$ to within constant multiplicative factors are, respectively,
\[
  \tilde{O}\Parens{\frac1{\pimin\gap}}
  \quad\text{and}\quad
  \tilde{O}\Parens{\frac1{\pimin\gap^3}}
  .
\]
By \cref{eq:mixing-time-bound},
the second of these is also a bound on the required sequence length to estimate $\tmix$.

The proof of \cref{thm:err} is based on analyzing the
convergence of the sample averages $\wh{\vM}$ and
$\hat\vpi$ to their expectation, and then using perturbation bounds
for eigenvalues to derive a bound on the error of $\hatgap$.
However, since these averages are formed using a \emph{single sample
path} from a (possibly) non-stationary Markov chain, we cannot use
standard large deviation bounds; moreover applying Chernoff-type
bounds for Markov chains to each entry of $\wh{\vM}$ would result in a
significantly worse sequence length requirement, roughly a factor of
$d$ larger.
Instead, we adapt probability tail bounds for sums of independent
random matrices~\citep{tropp2015intro} to our non-iid setting by
directly applying a blocking technique of~\citet{Bernstein27} as
described in the article of~\citet{Yu94}.
Due to ergodicity, the convergence rate can be bounded without any
dependence on the initial state distribution $\vpi^{(1)}$.
The proof of \cref{thm:err} is given in \cref{app:upper}.

Note that because the eigenvalues of $\vL$ are the same as that of the
transition probability matrix $\vP$, 
we could have instead opted to
estimate $\vP$, say, using simple frequency estimates obtained from
the sample path, and then computing the second largest eigenvalue of
this empirical estimate $\wh\vP$.
In fact, this approach is a way to extend to non-reversible chains, as
we would no longer rely on the symmetry of $\vM$ or $\vL$.
The difficulty with this approach is that $\vP$ lacks the structure
required by certain strong eigenvalue perturbation results.
One could instead invoke the Ostrowski-Elsner theorem
\citep[cf.~Theorem 1.4 on Page 170 of][]{stewart1990matrix}, which
bounds the \emph{matching distance} between the eigenvalues of a
matrix $\vA$ and its perturbation $\vA+\vE$ by $O(\norm{\vE}^{1/d})$.
Since $\norm{\wh\vP-\vP}$ is expected to be of size $O(n^{-1/2})$,
this approach will give a confidence interval for $\gap$ whose width
shrinks at a rate of $O(n^{-1/(2d)})$---an exponential slow-down
compared to the rate from \cref{thm:err}.
As demonstrated through an example from \citet{stewart1990matrix}, the
dependence on the $d$-th root of the norm of the perturbation cannot
be avoided in general.
Our approach based on estimating a symmetric matrix affords us the use
of perturbation results that exploit more structure.

Returning to the question of obtaining a fully empirical confidence 
interval for $\gap$ and $\pimin$, we notice that,
unfortunately, \cref{thm:err} falls short of being directly suitable for this,
at least without further assumptions.
This is because the deviation terms themselves depend inversely both
on $\gap$ and $\pimin$, and hence can never rule out $0$ (or an
arbitrarily small positive value) as a possibility for $\gap$ or
$\pimin$.\footnote{%
  Using \cref{thm:err}, it is possible to trap $\gap$ in the
  union of \emph{two} empirical confidence intervals---one around
  $\hatgap$ and the other around zero, both of which shrink in width
  as the sequence length increases.%
}
In effect, the fact that the Markov chain could be slow mixing and the
long-term frequency of some states could be small makes it difficult
to be confident in the estimates provided by $\hatgap$ and
$\hatpimin$.
This suggests that in order to obtain fully empirical confidence
intervals, we need an estimator that is not subject to such
effects---we pursue this in \cref{sec:empirical}.
\Cref{thm:err} thus primarily serves as a point of comparison
for what is achievable in terms of estimation accuracy when one does
not need to provide empirical confidence bounds.

\section{Fully empirical confidence intervals}\label{sec:empirical}
\begin{algorithm}
\caption{Empirical confidence intervals}
\label{alg:empest}
\begin{algorithmic}[1]
  \renewcommand\algorithmicrequire{\textbf{Input}:}
  \REQUIRE
    Sample path $(X_1,X_2,\dots,X_n)$,
    confidence parameter $\delta \in (0,1)$.

  \STATE Compute state visit counts and smoothed transition
  probability estimates:
  \begin{equation}
    \begin{aligned}
      N_i & :=
      \Abs{
        \Braces{
          t \in [n-1] : X_t = i
        }
      }
      , \quad i \in [d] ; \\
      N_{i,j} & :=
      \Abs{
        \Braces{
          t \in [n-1] : (X_t,X_{t+1}) = (i,j)
        }
      }
      , \quad (i,j) \in [d]^2 ; \\
      \wh{P}_{i,j}
      & :=
      \frac{N_{i,j} + 1/d}{N_i + 1}
      , \quad (i,j) \in [d]^2 .
    \end{aligned}
    \notag
  \end{equation}
  \label{step:P}

  \STATE Let $\giAh$ be the group inverse of $\wh\vA := \vI -
  \wh\vP$.
  \label{step:gi}

  \STATE Let $\hat\vpi \in \Delta^{d-1}$ be the unique stationary
  distribution for $\wh\vP$.
  \label{step:pi}

  \STATE Compute eigenvalues $\hat\lambda_1 {\geq} \hat\lambda_2
  {\geq} \dotsb {\geq} \hat\lambda_d$ of $\Sym(\wh\vL)$, where $\wh\vL
  := \Diag(\hat\vpi)^{1/2} \wh\vP \Diag(\hat\vpi)^{-1/2}$.
  \label{step:eig}

  \STATE Spectral gap estimate:
  \[ \hatgap := 1 - \max\braces{ \hat\lambda_2,\, |\hat\lambda_d| } . \]
  \label{step:gap}

  \STATE Empirical bounds for $|\wh{P}_{i,j}{-}P_{i,j}|$ for $(i,j){\in}[d]^2$:
  $c:=1.1$,
  $\emptail
    := \inf
    \braces{
      t\geq0 :
      2d^2 \parens{ 1 + \ceil{\log_c\frac{2n}{t}}_+ } e^{-t} \leq
      \delta
    }$,
  \begin{equation}
    \text{and} \quad
    \wh{B}_{i,j}
    :=
    \Parens{
      \sqrt{\frac{c\emptail}{2N_i}}
      + \sqrt{
        \frac{c\emptail}{2N_i}
        + 
        \sqrt{\frac{2c\wh{P}_{i,j}(1-\wh{P}_{i,j})\emptail}{N_i}}
        + \frac{(4/3)\emptail + \abs{\wh{P}_{i,j}-1/d}}{N_i}
      }
    }^2
    .
    \notag
  \end{equation}
  \label{step:P-bound}

  \STATE Relative sensitivity of $\vpi$:
  \begin{equation}
    \hat\kappa :=
    \frac12
    \max
    \Braces{
      \wh{A}_{j,j}^\# - \min\Braces{ \wh{A}_{i,j}^\# : i \in [d] }
      : j \in [d]
    } 
    .
    \notag
  \end{equation}
  \label{step:sens}

  \STATE Empirical bounds for $\max_{i \in [d]} |\hat{\pi}_i -
  \pi_i|$ and
  $\max\bigcup_{i\in[d]}
  \braces{
    \abs{\sqrt{\pi_i/\hat\pi_i}-1},\,
    \abs{\sqrt{\hat\pi_i/\pi_i}-1}
  }$:
  \begin{equation}
    \hat{b} := \hat\kappa \max\Braces{
      \wh{B}_{i,j}
      : (i,j) \in [d]^2
    }
    , \qquad
    \hat\rho := \frac12 \max \bigcup_{i\in[d]}
    \Braces{
      \frac{\hat{b}}{\hat\pi_i},\,
      \frac{\hat{b}}{\brackets{\hat\pi_i-\hat{b}}_+}
    }
    .
    \notag
  \end{equation}
  \label{step:pi-bound}

  \STATE Empirical bounds for $\abs{\hatgap-\gap}$:
  \begin{equation}
    \hat{w} := 2\hat\rho + \hat\rho^2
    + (1+2\hat\rho+\hat\rho^2)
    \Biggl(
      \sum_{(i,j)\in[d]^2} \frac{\hat\pi_i}{\hat\pi_j} \hat{B}_{i,j}^2
    \Biggr)^{1/2} .
    \notag
  \end{equation}
  \label{step:gap-bound}

\end{algorithmic}
\end{algorithm}

In this section, we address the shortcoming of \cref{thm:err} and give
fully empirical confidence intervals for the stationary probabilities
and the spectral gap $\gap$.
The main idea is to use the Markov property to eliminate the
dependence of the confidence intervals on the unknown quantities
(including $\pimin$ and $\gap$).
Specifically, we estimate the transition probabilities from the sample
path using simple frequency estimates: as a consequence of the Markov
property, for each state, the frequency estimates converge at a rate
that depends only on the number of visits to the state, and in
particular the rate (given the visit count of the state) is
independent of the mixing time of the chain.

As discussed in \cref{sec:rates}, it is possible to form a confidence
interval for $\gap$ based on the eigenvalues of an estimated
transition probability matrix by appealing to the
Ostrowski-Elsner theorem.
However, as explained earlier, this would lead to a slow
$O(n^{-1/(2d)})$ rate.
We avoid this slow rate by using an estimate of the symmetric matrix
$\vL$, so that we can use a stronger perturbation result (namely Weyl's
inequality, as in the proof of \cref{thm:err}) available for symmetric matrices.

To form an estimate of $\vL$ based on an estimate of the transition
probabilities, one possibility is to estimate $\vpi$ using a
frequency-based estimate for $\vpi$ as was done in \cref{sec:rates},
and appeal to the relation $\vL = \Diag(\vpi)^{1/2} \vP
\Diag(\vpi)^{-1/2}$ to form a plug-in estimate.
However, as noted in \cref{sec:rates-upper}, confidence intervals for
the entries of $\vpi$ formed this way may depend on the mixing time.
Indeed, such an estimate of $\vpi$ does not exploit the Markov
property.

We adopt a different strategy for estimating $\vpi$, which leads to
our construction of empirical confidence intervals, detailed in
\cref{alg:empest}.
We form the matrix $\wh\vP$ using smoothed frequency estimates of
$\vP$ (Step~\ref{step:P}), then compute the so-called group inverse
$\giAh$ of $\wh\vA = \vI - \wh\vP$ (Step~\ref{step:gi}), followed by
finding the unique stationary distribution $\hat\vpi$ of $\wh\vP$
(Step~\ref{step:pi}), this way decoupling the bound on the accuracy of $\hat\vpi$
from the mixing time.
The group inverse $\giAh$ of $\wh\vA$ is uniquely defined;
 and if
$\wh\vP$ defines an ergodic chain (which is the case here due to the
use of the smoothed estimates), $\giAh$ can be computed at the cost of
inverting an $(d{-}1){\times}(d{-}1)$ matrix~\citep[Theorem
5.2]{meyer1975role}.%
\footnote{%
\label{ftnote:group-inverse}
The group inverse of a square matrix $\vA$, a special case of the {\em Drazin inverse},
is the unique matrix $\vA^\#$ satisfying
$ \vA\vA^\#\vA = \vA$,
$\vA^\#\vA\vA^\# = \vA^\#$ and
$\vA^\#\vA = \vA\vA^\#$.
}
Further, once given $\giAh$, the unique stationary distribution
$\hat\vpi$ of $\wh\vP$ can be read out from the last row of
$\giAh$~\citep[Theorem 5.3]{meyer1975role}.
The group inverse is also be used to compute the sensitivity of
$\vpi$.
Based on $\hat\vpi$ and $\wh\vP$, we construct the plug-in estimate
$\wh\vL$ of $\vL$, and use the eigenvalues of its symmetrization to
form the estimate $\hatgap$ of the spectral gap (Steps~\ref{step:eig}
and~\ref{step:gap}).
In the remaining steps, 
we use perturbation analyses to relate $\hat\vpi$ and $\vpi$, viewing
$\vP$ as the perturbation of $\wh\vP$; and also to relate $\hatgap$
and $\gap$, viewing $\vL$ as a perturbation of $\Sym(\wh\vL)$.
Both analyses give error bounds entirely in terms of observable
quantities (e.g., $\hat\kappa$), tracing back to empirical error
bounds for the smoothed frequency estimates of $\vP$.

The most computationally expensive step in \cref{alg:empest} is the
computation of the group inverse $\giAh$, which, as noted reduces to
matrix inversion.
Thus, with a standard implementation of matrix inversion, the
algorithm's time complexity is $O(n + d^3)$, while its space
complexity is $O(d^2)$.

To state our main theorem concerning \cref{alg:empest}, we first define
$\kappa$ to be analogous to $\hat{\kappa}$ from Step~\ref{step:sens},
with $\giAh$ replaced by the group inverse $\giA$ of $\vA := \vI -
\vP$.
The result is as follows.
\begin{theorem}
  \label{thm:empirical}
  Suppose \cref{alg:empest} is given as input a sample path of length
  $n$ from an ergodic and reversible Markov chain and confidence
  parameter $\delta \in (0,1)$.
  Let $\gap>0$ denote the spectral gap, $\vpi$ the unique stationary
  distribution, and $\pimin>0$ the minimum stationary probability.
  Then, on an event of probability at least $1-\delta$,
  \[
    \pi_i \in [\hat\pi_i-\hat{b},\hat\pi_i+\hat{b}]
    \quad \text{for all $i \in [d]$} ,
    \qquad\text{and}\qquad
    \gap \in [\hatgap-\hat{w},\hatgap+\hat{w}]
    .
  \]
  Moreover, $\hat{b}$ and $\hat{w}$ almost surely satisfy (as $n \to
  \infty$)
  \[
    \hat{b}
    =
    O\Parens{
      \max_{(i,j) \in [d]^2}
      \kappa
      \sqrt{\frac{P_{i,j}\log\log n}{\pi_i n}}
    }
    ,
    \quad
    \hat{w}
    =
    O\Parens{ 
      \frac{\kappa}{\pimin} \sqrt{\frac{\log\log n}{\pimin n}} + 
      \sqrt{ \frac{d\log\log n}{ \pimin n}}
    }
    .\footnote{%
    In \cref{thm:empirical,thm:combined}, our use of big-$O$ notation
    is as follows.
    For a random sequence $(Y_n)_n$ and a (non-random) positive
    sequence $(\veps_{\theta,n})_n$ parameterized by $\theta$, we say
    ``$Y_n = O(\veps_{\theta,n})$ holds almost surely as
    $n\to\infty$'' if there is some universal constant $C>0$ such that
    for all $\theta$, $\limsup_{n\to\infty} Y_n/\veps_{\theta,n} \leq
    C$ holds almost surely.%
  }%
  \]
\end{theorem}
The proof of \cref{thm:empirical} is given in \cref{app:empirical}.
As mentioned above, the obstacle encountered in \cref{thm:err} is
avoided by exploiting the Markov property.
We establish fully observable upper and lower bounds on the entries of
$\vP$ that converge at a $\sqrt{n/\log\log n}$ rate using standard
martingale tail inequalities; this justifies the validity of the
bounds from Step~\ref{step:P-bound}.
Properties of the group inverse~\citep{meyer1975role,cho2001comparison} and eigenvalue
perturbation theory~\citep{stewart1990matrix} are used to validate the
empirical bounds on $\pi_i$ and $\gap$ developed in the remaining
steps of the algorithm.

The first part of \cref{thm:empirical} provides valid empirical
confidence intervals for each $\pi_i$ and for $\gap$, which are
simultaneously valid at confidence level $\delta$.
The second part of \cref{thm:empirical} shows that the width of the
intervals decrease as the sequence length increases.
We show in \cref{sec:asymptotic} that
$\kappa \le d/\gap$, and hence
\[
  \hat{b}
  =
  O\Parens{
    \max_{(i,j) \in [d]^2}
    \frac{d}{\gap}
    \sqrt{\frac{P_{i,j}\log\log n}{\pi_i n}}
  }
  , \quad
  \hat{w}
  =
  O\Parens{ 
    \frac{d}{\pimin\gap} \sqrt{\frac{\log\log n}{\pimin n}}
  }
  .
\] 

It is easy to combine \cref{thm:err,thm:empirical} to yield intervals
whose widths shrink at least as fast as both the non-empirical
intervals from \cref{thm:err} and the empirical intervals from
\cref{thm:empirical}.
Specifically, determine lower bounds on $\pimin$ and $\gap$ using
\cref{alg:empest},
\[
  \pimin \geq \min_{i \in [d]} \brackets{ \hat\pi_i - \hat{b} }_+
  , \quad
  \gap \geq \brackets{ \hatgap - \hat{w} }_+
  ;
\]
then plug-in these lower bounds for $\pimin$ and $\gap$ in the
deviation bounds in \cref{eq:gapbound} from \cref{thm:err}.
This yields a new interval centered around the estimate of $\gap$ from
\cref{thm:err}, and it no longer depends on unknown quantities.
The interval is a valid $1-2\delta$ probability confidence interval
for $\gap$, and for sufficiently large $n$, the width shrinks at the
rate given in \cref{eq:gapbound}.
We can similarly construct an empirical confidence interval for
$\pimin$ using \cref{eq:piminbound}, which is valid on the same
$1-2\delta$ probability event.\footnote{%
  For the $\pimin$ interval, we only plug-in lower bounds on $\pimin$
  and $\gap$ only where these quantities appear as $1/\pimin$ and
  $1/\gap$ in \cref{eq:piminbound}.
  It is then possible to ``solve'' for observable bounds on $\pimin$.
  See \cref{app:combined} for details.%
}
Finally, we can take the intersection of these new intervals with the
corresponding intervals from \cref{alg:empest}.
This is summarized in the following \lcnamecref{thm:combined}, which
we prove in \cref{app:combined}.
\begin{theorem}
  \label{thm:combined}
  The following holds under the same conditions as
  \cref{thm:empirical}.
  For any $\delta \in (0,1)$, the confidence intervals $\wh{U}$ and
  $\wh{V}$ described above for $\pimin$ and $\gap$, respectively,
  satisfy $\pimin \in \wh{U}$ and $\gap \in \wh{V}$ with probability
  at least $1-2\delta$.
  Furthermore, the widths of these intervals almost surely satisfy
  (as $n \to \infty$) 
  \[
    |\wh{U}|
    =
    O\Parens{
      \sqrt{\frac{\pimin\log\frac{d}{\pimin\delta}}{\gap n}}
    }
    ,
    \quad
    |\wh{V}|
    =
    O\Parens{
      \min\Braces{
        \sqrt{\frac{\log\frac{d}{\delta}\cdot\log(n)}{\pimin\gap n}}
        ,\,
        \hat{w}
      }
    }
  \]
  where $\hat{w}$ is the width from \cref{alg:empest}.
\end{theorem}

\section{Discussion}\label{sec:discussion}
The construction used in \cref{thm:combined} applies more generally:
Given a confidence interval of the form $I_n = I_n(\gap,\pimin,\delta)$ 
for some confidence level $\delta$
and a fully empirical confidence set $E_n(\delta)$ for $(\gap,\pimin)$  for the same level,
$I_n' = E_n(\delta) \cap \cup_{(\gamma,\pi)\in E_n(\delta)} I_n(\gamma,\pi,\delta)$ is a valid
fully empirical $2\delta$-level confidence interval whose asymptotic width
matches that of $I_n$ up to lower order terms under reasonable assumptions on $E_n$ and $I_n$.
In particular, this suggests that future work should focus on 
closing the gap between the lower and upper bounds on the accuracy
of point-estimation. Another interesting direction is to 
reduce the computation cost: The current cubic cost in the number of states
can be too high even when the number of states is only moderately
large.

Perhaps more important, however, is to extend 
our results to large state space Markov chains:
In most practical applications the state space is continuous
or is exponentially large in some natural parameters.
As follows from our lower bounds, without further assumptions,
the problem of fully data dependent estimation of the mixing time
is intractable for information theoretical reasons.
Interesting directions for future work thus must consider Markov
chains with specific structure. Parametric classes of Markov chains,
including but not limited to Markov chains with factored transition kernels
with a few factors, are a promising candidate for such future investigations.
The results presented here are a first step in the ambitious research agenda
outlined above, and we hope that they will
serve as a point of departure
for
further insights
in the area of fully empirical estimation of Markov chain 
parameters based on a single sample path.

\bibliography{all}
\bibliographystyle{plain}

\appendix

\section{Proofs of the lower bounds}\label{app:lower}
\begin{theorem}[Theorem~\ref{thm:lb-pimin} restated]
  Pick any $\bar\pi \in (0,1/4)$.
  Consider any estimator $\hatgap$ that takes as input a random sample
  path of length $n \leq 1/(4\bar\pi)$ from a Markov chain starting
  from any desired initial state distribution.
  There exists a two-state ergodic and reversible Markov chain
  distribution with spectral gap $\gap \geq 1/2$ and minimum
  stationary probability $\pimin \geq \bar\pi$ such that
  \[
    \Pr\Brackets{ |\hatgap - \gap| \geq 1/8 } \geq 3/8 .
  \]
\end{theorem}
\begin{proof}
  Fix $\bar\pi \in (0,1/4)$.
  Consider two Markov chains given by the following stochastic
  matrices:
  \[
    \vP^{(1)} :=
    \begin{bmatrix}
      1-\bar\pi & \bar\pi \\
      1-\bar\pi & \bar\pi
    \end{bmatrix}
    , \quad
    \vP^{(2)} :=
    \begin{bmatrix}
      1-\bar\pi & \bar\pi \\
      1/2 & 1/2
    \end{bmatrix}
    .
  \]
  Each Markov chain is ergodic and reversible; their stationary
  distributions are, respectively, $\vpi^{(1)} = (1-\bar\pi,\bar\pi)$
  and $\vpi^{(2)} = (1/(1+2\bar\pi),2\bar\pi/(1+2\bar\pi))$.
  We have $\pimin \geq \bar\pi$ in both cases.
  For the first Markov chain, $\slem = 0$, and hence the spectral gap
  is $1$; for the second Markov chain, $\slem = 1/2-\bar\pi$, so the
  spectral gap is $1/2+\bar\pi$.

  In order to guarantee $|\hatgap - \gap| < 1/8 < |1 -
  (1/2+\bar\pi)|/2$, it must be possible to distinguish the two Markov
  chains.
  Assume that the initial state distribution has mass at least $1/2$
  on state $1$.
  (If this is not the case, we swap the roles of states $1$ and $2$ in
  the constructions above.)
  With probability at least half, the initial state is $1$; and both
  chains have the same transition probabilities from state $1$.
  The chains are indistinguishable unless the sample path eventually
  reaches state $2$.
  But with probability at least $3/4$, a sample path of length $n <
  1/(4\bar\pi)$ starting from state $1$ always remains in the same
  state (this follows from properties of the geometric distribution
  and the assumption $\bar\pi < 1/4$).
\end{proof}

\begin{theorem}[Theorem~\ref{thm:lb-gap} restated]
  There is an absolute constant $c>0$ such that the following holds.
  Pick any positive integer $d \geq 3$ and any $\bar\gamma \in
  (0,1/2)$.
  Consider any estimator $\hatgap$ that takes as input a random sample
  path of length $n < c d\log(d) / \bar\gamma$ from a $d$-state
  reversible Markov chain starting from any desired initial state
  distribution.
  There is an ergodic and reversible Markov chain distribution
  with spectral gap $\gap \in [\bar\gamma,2\bar\gamma]$ and minimum
  stationary probability $\pimin \geq 1/(2d)$ such that
  \[
    \Pr\Brackets{ |\hatgap - \gap| \geq \bar\gamma/2} \geq 1/4 .
  \]
\end{theorem}
\begin{proof}
  We consider $d$-state Markov chains of the following form:
  \[
    P_{i,j} =
    \begin{cases}
      1-\veps_i & \text{if $i = j$} ; \\
      \displaystyle\frac{\veps_i}{d-1} & \text{if $i \neq j$}
    \end{cases}
  \]
  for some $\veps_1, \veps_2, \dotsc, \veps_d \in (0,1)$.
  Such a chain is ergodic and reversible, and its unique stationary
  distribution $\vpi$ satisfies
  \[
    \pi_i = \frac{1/\veps_i}{\sum_{j=1}^d 1/\veps_j}
    .
  \]
  We fix $\veps := \frac{d-1}{d/2}\bar\gamma$ and set $\veps' :=
  \frac{d/2-1}{d-1} \veps < \veps$.
  Consider the following $d+1$ different Markov chains of the type
  described above:
  \begin{itemize}
    \item
      $\vP^{(0)}$: $\veps_1 = \dotsb = \veps_d = \veps$.
      For this Markov chain, $\lambda_2 = \lambda_d = \slem =
      1-\frac{d}{d-1}\veps$.

    \item
      $\vP^{(i)}$ for $i \in [d]$: $\veps_j = \veps$ for $j \neq i$,
      and $\veps_i = \veps'$.
      For these Markov chains, $\lambda_2 = 1 - \veps' -
      \frac{1}{d-1}\veps = 1 - \frac{d/2}{d-1} \veps$,
      and $\lambda_d = 1 - \frac{d}{d-1} \veps$.
      So $\slem = 1-\frac{d/2}{d-1} \veps$.

  \end{itemize}
  The spectral gap in each chain satisfies $\gap \in
  [\bar\gamma,2\bar\gamma]$; in $\vP^{(i)}$ for $i \in [d]$, it is half
  of what it is in $\vP^{(0)}$.
  Also $\pi_i \geq 1/(2d)$ for each $i \in [d]$.

  In order to guarantee $|\hatgap - \gap| < \bar\gamma/2$, it must
  be possible to distinguish $\vP^{(0)}$ from each $\vP^{(i)}$,
  $i\in[d]$.
  But $\vP^{(0)}$ is identical to $\vP^{(i)}$ except for the transition
  probabilities from state $i$.
  Therefore, regardless of the initial state, the sample path must
  visit all states in order to distinguish $\vP^{(0)}$ from each
  $\vP^{(i)}$, $i \in [d]$.
  For any of the $d+1$ Markov chains above, the earliest time in which
  a sample path visits all $d$ states
  stochastically dominates a generalized coupon collection time $T = 1 +
  \sum_{i=1}^{d-1} T_i$, where $T_i$ is the number of steps required to
  see the $(i+1)$-th distinct state in the sample path beyond the first
  $i$.
  The random variables $T_1,T_2,\dotsc,T_{d-1}$ are independent, and are
  geometrically distributed, $T_i \sim \Geom(\veps -
  (i-1)\veps/(d-1))$.
  We have that
  \[
    \bbE[T_i] = \frac{d-1}{\veps(d-i)} , \quad
    \var(T_i) = \frac{1 - \veps \frac{d-i}{d-1}}{%
      \Parens{ \veps \frac{d-i}{d-1} }^2
    }
    .
  \]
  Therefore
  \[
    \bbE[T] = 1 + \frac{d-1}{\veps} H_{d-1} , \quad
    \var(T) \leq \Parens{ \frac{d-1}{\veps} }^2 \frac{\pi^2}{6} 
  \]
  where $H_{d-1} = 1 + 1/2 + 1/3 + \dotsb + 1/(d-1)$.
  By the Paley-Zygmund inequality,
  \[
    \Pr\Parens{ T > \frac13\bbE[T] }
    \geq \frac{1}{1 + \frac{\var(T)}{(1-1/3)^2\bbE[T]^2}}
    \geq \frac{1}{1 + \frac{\Parens{ \frac{d-1}{\veps} }^2
    \frac{\pi^2}{6}}{(4/9)\Parens{ \frac{d-1}{\veps} H_2}^2}}
    \geq \frac14
    .
  \]
  Since $n < c d\log(d) / \bar\gamma \leq (1/3) (1 + (d-1) H_{d-1} /
  (2\bar\gamma)) = \bbE[T]/3$ (for an appropriate absolute constant
  $c$), with probability at least $1/4$, the sample path does not
  visit all $d$ states.
\end{proof}

We claim in Section~\ref{sec:intro} that a sample path of length
$\Omega\parens{ (d \log d)/\gap + 1/\pimin}$ is required to guarantee
constant multiplicative accuracy in estimating $\gap$.
This follows by combining Theorems~\ref{thm:lb-pimin}
and~\ref{thm:lb-gap} in a standard, straightforward way.
Specifically, if the length of the sample path $n$ is smaller than
$(n_1 + n_2) / 2$---where $n_1$ is the lower bound from
Theorem~\ref{thm:lb-pimin}, and $n_2$ is the lower bound from
Theorem~\ref{thm:lb-gap}---then $n$ is smaller than
$\max\braces{n_1,n_2}$.
So at least one of Theorem~\ref{thm:lb-pimin} and
Theorem~\ref{thm:lb-gap} implies the existence of an ergodic and
reversible Markov chain distribution with spectral gap $\gap$ and
stationary distribution minorized by $\pimin$ such that
\[
  \Pr\Brackets{ |\hatgap - \gap| \geq \gap/4 } \geq 1/4
  .
\]

\section{Proof of \cref{thm:err}}\label{app:upper}

In this section, we prove \cref{thm:err}.

\subsection{Accuracy of $\hatpimin$}

We start by proving the deviation bound on $\pimin-\hatpimin$, from
which we may easily deduce \cref{eq:piminbound} in \cref{thm:err}.
\begin{lemma}
  \label{lem:hatpimin}
  Pick any $\delta \in (0,1)$, and let
  \begin{equation}
    \veps_n :=
    \frac{ \ln\Parens{\frac{d}\delta\sqrt{\frac{2}{\pimin}}} }{\gap n}
    .
    \label{eq:singleton-veps}
  \end{equation}
  With probability at least $1-\delta$, the following inequalities
  hold simultaneously:
  \begin{align}
    \Abs{\hat{\pi}_i-\pi_i}
    & \le
    \sqrt{8\pi_i(1-\pi_i) \veps_n}
    + 20 \veps_n
    \quad \text{for all $i \in [d]$}
    ;
    \label{eq:singletonboundsimple}
    \\
    \Abs{\hatpimin - \pimin}
    & \leq
    4\sqrt{\pimin\veps_n}
    + 47\veps_n
    .
    \label{eq:hatpimin}
  \end{align}
\end{lemma}
\begin{proof}
  We use the following Bernstein-type inequality for Markov chains
  from \citet[][Theorem~3.8]{paulin15}: letting $\bbP^{\vpi}$ denote
  the probability with respect to the stationary chain (where the
  marginal distribution of each $X_t$ is $\vpi$), we have for every
  $\eps>0$,
  \[
    \bbP^{\vpi}
    \Parens{
      \abs{ \hat\pi_i-\pi_i} > \eps
    } \le
    2\exp\Parens{
      -\frac{
        n\gap\eps^2
      }{
        4\pi_i(1-\pi_i)+10\eps
      }
    },
    \qquad i\in[d]
    .
  \]
  To handle possibly non-stationary chains, as is our case, we combine
  the above inequality with \citet[][Proposition 3.14]{paulin15}, to
  obtain for any $\eps>0$,
  \begin{align*}
    \bbP
    \Parens{
      \abs{ \hat\pi_i-\pi_i} > \eps
    }
    &\le
    \sqrt{
      \frac1\pimin
      \bbP^{\vpi}
      \Parens{
        \abs{ \hat\pi_i-\pi_i} > \eps
      }
    } 
    \le
    \sqrt{\frac2\pimin}
    \exp\Parens{
      -\frac{
        n\gap\eps^2
      }{
        8\pi_i(1-\pi_i)+20\eps
      }
    }.
  \end{align*}
  Using this tail inequality with $\eps :=
  \sqrt{8\pi_i(1-\pi_i)\veps_n} + 20\veps_n$ and a union bound over
  all $i \in [d]$ implies that the inequalities
  in~\cref{eq:singletonboundsimple} hold with probability at least
  $1-\delta$.

  Now assume this $1-\delta$ probability event holds; it remains to
  prove that \cref{eq:hatpimin} also holds in this event.
  Without loss of generality, we assume that $\pimin = \pi_1 \leq
  \pi_2 \leq \dotsb \leq \pi_d$.
  Let $j \in [d]$ be such that $\hatpimin = \hat\pi_j$.
  By \cref{eq:singletonboundsimple}, we have $|\pi_i - \hat\pi_i| \leq
  \sqrt{8\pi_i\veps_n} + 20\veps_n$ for each $i \in \{1,j\}$.
  Since $\hatpimin \leq \hat\pi_1$,
  \[
    \hatpimin - \pimin
    \leq \hat\pi_1 - \pi_1
    \leq \sqrt{8\pimin\veps_n} + 20\veps_n
    \leq \pimin + 22\veps_n
  \]
  where the last inequality follows by the AM/GM inequality.
  Furthermore, using the fact that $a \leq b\sqrt{a} + c \Rightarrow a
  \leq b^2 + b\sqrt{c} + c$ for nonnegative numbers $a, b, c \geq
  0$~\citep[see, e.g.,][]{BBL04} with the inequality $\pi_j \leq
  \sqrt{8\veps_n} \sqrt{\pi_j} + \parens{\hat\pi_j + 20\veps_n}$ gives
  \[
    \pi_j
    \leq 
    \hat\pi_j
    + \sqrt{8(\hat\pi_j + 20\veps_n)\veps_n}
    + 28\veps_n
    .
  \]
  Therefore
  \[
    \pimin - \hatpimin
    \leq \pi_j - \hat\pi_j
    \leq \sqrt{8(\hatpimin + 20\veps_n)\veps_n} + 28\veps_n
    \leq \sqrt{8(2\pimin + 42\veps_n)\veps_n} + 28\veps_n
    \leq 4\sqrt{\pimin\veps_n} + 47\veps_n
  \]
  where the second-to-last inequality follows from the above bound on
  $\hatpimin - \pimin$, and the last inequality uses $\sqrt{a+b} \leq
  \sqrt{a} + \sqrt{b}$ for nonnegative $a,b \geq 0$.
\end{proof}


\subsection{Accuracy of $\hatgap$}

Let us now turn to proving \cref{eq:gapbound}, i.e., 
the bound on the error of the spectral gap estimate $\hatgap$.
The accuracy of $\hatgap$ is based on the accuracy of $\Sym(\wh\vL)$
in approximating $\vL$ via Weyl's inequality:
\begin{align*}
  |\hat\lambda_i - \lambda_i|
  \leq \norm{\Sym(\wh\vL) - \vL}
  \quad \text{for all }i \in [d] .
\end{align*}
Moreover, the triangle inequality implies that symmetrizing $\wh\vL$
can only help:
\begin{align*}
  \norm{\Sym(\wh\vL) - \vL} \leq \norm{\wh\vL - \vL} .
\end{align*}
Therefore, we can deduce \cref{eq:gapbound} in \cref{thm:err} from the
following lemma.
\begin{lemma}
  \label{lem:gap}
  There exists an absolute constant $C>0$ such that the following
  holds.
  For any $\delta \in (0,1)$, with probability at least $1-\delta$,
  the bounds from \cref{lem:hatpimin} hold, and
  \begin{equation*}
    \norm{\wh\vL - \vL}
    \leq
    C\,
      \sqrt{
        \frac{
          \log\Parens{\frac{d}{\delta}}
          \log\Parens{\frac{n}{\pimin\delta}}
        }{\pimin\gap n}
      }
      +
      C\,
      \frac{\log\Parens{\frac{1}{\gap}}}{\gap n}
  \end{equation*}
\end{lemma}
The remainder of this section is devoted to proving this
\lcnamecref{lem:gap}.

The error $\wh\vL - \vL$ may be written as
\[
  \wh\vL - \vL
  = \errm + \errp \vL + \vL \errp + \errp \vL \errp
  + \errp \errm + \errm \errp + \errp \errm \errp\,,
\]
where
\begin{align*}
  \errp & := \Diag(\hat\vpi)^{-1/2} \Diag(\vpi)^{1/2} - \vI , \\
  \errm & := \Diag(\vpi)^{-1/2} \Parens{
    \wh\vM - \vM
  } \Diag(\vpi)^{-1/2} .
\end{align*}
Therefore
\begin{align*}
  \norm{\wh\vL - \vL}
  \leq \norm{\errm} +
  \Parens{
    \norm{\errm} + \norm{\vL}
  }
  \Parens{
    2\norm{\errp} + \norm{\errp}^2
  }
  .
\end{align*}
Since $\norm{\vL} \leq 1$ and $\norm{\wh\vL} \leq 1$~\citep[Lemma
12.1]{LePeWi08}, we also have $\norm{\wh\vL - \vL} \leq 2$.
Therefore,
\begin{align}
  \norm{\wh\vL - \vL}
  \leq
  \min\Braces{
    \norm{\errm} +
    \Parens{
      \norm{\errm} + \norm{\vL}
    }
    \Parens{
      2\norm{\errp} + \norm{\errp}^2
    }
    ,\,
    2
  }
  \leq 3(\norm{\errm}+\norm{\errp})
  .
  \label{eq:vlsimple}
\end{align}

\subsection{A bound on $\norm{\errp}$}

Since $\errp$ is diagonal,
\begin{align*}
  \norm{\errp}
  = \max_{i\in[d]}
  \Abs{
    \sqrt{\frac{\pi_i}{\hat\pi_i}} - 1
  }
  .
\end{align*}
Assume that
\begin{equation}
  n \geq \frac{108 \ln\Parens{\frac{d}\delta\sqrt{\frac{2}{\pimin}}}
  }{\pimin\gap}
  ,
  \label{eq:min-n1}
\end{equation}
in which case
\[
  \sqrt{8\pi_i(1-\pi_i) \veps_n} + 20\veps_n
  \leq \frac{\pi_i}{2}
\]
where $\veps_n$ is as defined in \cref{eq:singleton-veps}.
Therefore, in the $1-\delta$ probability event from
\cref{lem:hatpimin}, we have $|\pi_i - \hat\pi_i| \leq \pi_i/2$ for
each $i \in [d]$, and moreover, $2/3 \leq \pi_i/\hat\pi_i \leq 2$ for
each $i \in [d]$.
For this range of $\pi_i/\hat\pi_i$, we have
\[
  \Abs{ \sqrt{\frac{\pi_i}{\hat\pi_i}} - 1 }
  \leq 
  \Abs{ \frac{\hat\pi_i}{\pi_i} - 1 }
  .
\]
We conclude that if $n$ satisfies \cref{eq:min-n1}, then in this
$1-\delta$ probability event from \cref{lem:hatpimin},
\begin{align}
  \norm{\errp}
  & \leq
  \max_{i \in [d]}
  \Abs{
    \frac{\hat\pi_i}{\pi_i} - 1
  }
  \leq
  \max_{i \in [d]}
  \frac{
    \sqrt{8\pi_i(1-\pi_i) \veps_n} + 20 \veps_n
  }{\pi_i}
  \notag \\
  & \leq
  \sqrt{\frac{8\veps_n}{\pimin}} + \frac{20\veps_n}{\pimin}
  =
  \sqrt{
    \frac{
      8\ln\Parens{\frac{d}\delta\sqrt{\frac{2}{\pimin}}}
    }{
      \pimin\gap n
    }
  }
  +
  \frac{
    20\ln\Parens{\frac{d}\delta\sqrt{\frac{2}{\pimin}}}
  }{
    \pimin\gap n
  }
  .
  \label{eq:singletonbound}
\end{align}

\subsection{Accuracy of doublet frequency estimates (bounding $\norm{\errm}$)}
\label{sec:pairwise}
In this section we prove a bound on $\norm{\errm}$.
For this, we decompose
$\errm = \Diag(\vpi)^{-1/2}(\wh\vM-\vM)\Diag(\vpi)^{-1/2} $
into $\EE{\errm}$ and 
$\errm-\EE{\errm}$, the first measuring the effect of a non-stationary start of the chain,
while the second measuring the variation due to randomness.

\subsubsection{Bounding $\norm{\EE{\errm}}$: The price of a non-stationary start.}
Let $\vpi^{(t)}$ be the distribution of states at time step $t$.
We will make use of the following proposition, which can be derived by
following \citet[Proposition 1.12]{MoTe06}:
\begin{proposition}
For $t\ge 1$, let $\vUpsilon^{(t)}$ be the vector with $\Upsilon^{(t)}_i= \frac{\pi_i^{(t)}}{\pi_i}$ and let
$\norm{\cdot}_{2,\vpi}$ denote the $\vpi$-weighted $2$-norm
\begin{align}
\label{eq:vv}
  \norm{\vv}_{2,\vpi} := \Parens{ \sum_{i=1}^d \pi_i v_i^2 }^{1/2}
  .
\end{align}
 Then,
\begin{equation}
  \label{eq:L2-contraction}
  \norm{ \vUpsilon^{(t)} - \v1 }_{2,\vpi} \leq
  \frac{(1-\gap)^{t-1}}{\sqrt{\pimin}} \,.
\end{equation}
\end{proposition}
An immediate corollary of this result is that 
\begin{align}
\label{eq:dvpi-ratio-bound}
\Norm{  \Dvpit \Dvpi^{-1} - \vI  } \le \frac{(1-\gap)^{t-1}}{\pimin}\,.
\end{align}
Now note that
\begin{align*}
\bbE (\wh\vM) = \frac{1}{n-1} \sum_{t=1}^{n-1} \Dvpit \vP\,
\end{align*}
and thus
\begin{align*}
\EE{\errm} &=
  \Dvpi^{-1/2}
  \left(\bbE( \wh\vM )- \vM\right)
  \Dvpi^{-1/2}\\
& =  
  \frac{1}{n-1} \sum_{t=1}^{n-1}
  \Dvpi^{-1/2} (\Dvpit  - \Dvpi) \vP \Dvpi^{-1/2} \\
& =  
  \frac{1}{n-1} \sum_{t=1}^{n-1}
  \Dvpi^{-1/2} (\Dvpit \Dvpi^{-1} - \vI) \vM \Dvpi^{-1/2} \\
& =  
  \frac{1}{n-1} \sum_{t=1}^{n-1}
   (\Dvpit \Dvpi^{-1} - \vI) \vL \,.
\end{align*}
Combining this, $\norm{\vL}\le 1$ and~\cref{eq:dvpi-ratio-bound}, we get
\begin{align}
  \norm{\bbE(\errm)}
  \le \frac{1}{(n-1)\pimin} \sum_{t=1}^{n-1} (1-\gap)^{t-1}
  \le \frac{1}{(n-1)\gap\pimin}
  .
  \label{eq:biasbound}
\end{align}

\subsubsection{
Bounding $\norm{\errm-\EE{\errm}}$: 
Application of a matrix tail inequality}
\label{sec:pairwise-tailbound}

In this section we analyze the deviations of $\errm-\EE{\errm}$. By the definition of $\errm$,
\begin{align}
\label{eq:errmdev1}
\norm{\errm-\EE{\errm}}
 =  \norm{\Diag(\vpi)^{-1/2}\Parens{\wh\vM - \bbE \wh\vM } \Diag(\vpi)^{-1/2}}
 \,.
\end{align}
The matrix $\wh\vM - \EE{\wh\vM}$  is defined 
as a sum of dependent centered random matrices.
We will use the blocking technique of \citet{Bernstein27} 
to relate the likely deviations of this matrix to that of
a sum of independent centered random matrices.
The deviations of these will then bounded
with the help of a Bernstein-type matrix tail inequality due to \citet{tropp2015intro}.

We divide $[n-1]$ into contiguous blocks of time steps; each has size
$a \leq n/3$ except possibly the first block, which has size between
$a$ and $2a-1$.
Formally, let $a' := a + ((n-1) \bmod a) \leq 2a-1$, and define
\begin{align*}
  F & := [a'] , \\
  H_s & := \{ t \in [n-1] : a' + 2(s-1)a + 1 \leq t \leq a' + (2s-1)a \} , \\
  T_s & := \{ t \in [n-1] : a' + (2s-1)a + 1 \leq t \leq a' + 2sa \} ,
\end{align*}
for $s=1,2,\dotsc$.
Let $\mu_H$ (resp., $\mu_T$) be the number of non-empty $H_s$ (resp., $T_s$)
blocks.
Let $n_H := a\mu_H$ (resp., $n_T := a\mu_T$) be the number of time
steps in $\cup_s H_s$ (resp., $\cup_s T_s$).
We have
\begin{align}
  \wh\vM
  & = \frac1{n-1} \sum_{t=1}^{n-1} \ve_{X_t} \ve_{X_{t+1}}^\t
  \notag \\
  & = \frac{a'}{n-1} \cdot
    \underbrace{
      \frac{1}{a'} \sum_{t\in F} \ve_{X_t} \ve_{X_{t+1}}^\t
     }_{\wh\vM_F} +
  \frac{n_H}{n-1} \cdot
  \underbrace{
    \frac1{\mu_H}
    \sum_{s=1}^{\mu_H}
    \Parens{
      \frac1a \sum_{t \in H_s} \ve_{X_t} \ve_{X_{t+1}}^\t
    }
  }_{\wh{\vM}_H}
  \notag \\
  & \qquad
  + \frac{n_T}{n-1} \cdot
  \underbrace{
    \frac1{\mu_T}
    \sum_{s=1}^{\mu_T}
    \Parens{
      \frac1a \sum_{t \in T_s} \ve_{X_t} \ve_{X_{t+1}}^\t
    }
  }_{\wh{\vM}_T}
  .
  \label{eq:blocking}
\end{align}
Here, $\ve_i$ is the $i$-th coordinate basis vector, so $\ve_i
\ve_j^\t \in \{0,1\}^{d \times d}$ is a $d \times d$ matrix of all
zeros except for a $1$ in the $(i,j)$-th position.

The contribution of the first block is easily bounded using the
triangle inequality:
\begin{multline}
  \frac{a'}{n-1}
  \Norm{
    \Diag(\vpi)^{-1/2}
    \Parens{
      \wh\vM_F - \bbE(\wh\vM_F)
    }
    \Diag(\vpi)^{-1/2}
  }
  \\
  \leq
  \frac1{n-1} \sum_{t \in F}
  \Braces{
    \Norm{
      \frac{
        \ve_{X_t} \ve_{X_{t+1}}^\t
      }{\sqrt{\pi_{X_t}\pi_{X_{t+1}}}}
    }
    +
    \Norm{
      \bbE\Parens{
        \frac{
          \ve_{X_t} \ve_{X_{t+1}}^\t
        }{\sqrt{\pi_{X_t}\pi_{X_{t+1}}}}
      }
    }
  }
  \leq
  \frac{2a'}{\pimin(n-1)}
  .
  \label{eq:first-block}
\end{multline}

It remains to bound the contributions of the $H_s$ blocks and the
$T_s$ blocks.
We just focus on the the $H_s$ blocks, since the analysis is identical
for the $T_s$ blocks.

Let
\[
  \vY_s := \frac1a \sum_{t \in H_s} \ve_{X_t} \ve_{X_{t+1}}^\t ,
  \quad s \in [\mu_H] ,
\]
so
\[
  \wh{\vM}_H = \frac1{\mu_H} \sum_{s=1}^{\mu_H} \vY_s ,
\]
an average of the random matrices $\vY_s$.
For each $s \in [\mu_H]$, the random matrix $\vY_s$ is a function of
\[ (X_t : a' + 2(s-1)a + 1 \leq t \leq a' + (2s-1)a + 1) \]
(note the $+1$ in the upper limit of $t$),
so $\vY_{s+1}$ is $a$ time steps ahead of $\vY_s$.
When $a$ is sufficiently large, we will be able to effectively treat
the random matrices $\vY_s$ as if they were independent.
In the sequel, we shall always assume that the block length $a$
satisfies
\begin{equation}
  \label{eq:block-length}
  a \geq
  a_\delta
  :=
  \frac{1}{\gap} \ln \frac{2(n-2)}{\delta\pimin}
\end{equation}
for $\delta \in (0,1)$.

Define
\begin{align*}
  \vpi^{(H_s)}  := \frac1a \sum_{t \in H_s} \vpi^{(t)} , \qquad \qquad
  \vpi^{(H)}  := \frac1{\mu_H} \sum_{s=1}^{\mu_H} \vpi^{(H_s)} .
\end{align*}
Observe that
\[
  \bbE(\vY_s)
  = \Diag(\vpi^{(H_s)}) \vP
\]
so
\[
  \bbE\Parens{
    \frac1{\mu_H} \sum_{s=1}^{\mu_H} \vY_s
  } = \Diag(\vpi^{(H)}) \vP
  .
\]

Define
\[
  \vZ_s
  := \Diag(\vpi)^{-1/2} \left( \vY_s 
  - \bbE(\vY_s) \right)\Diag(\vpi)^{-1/2}
  .
\]
We apply a matrix tail inequality to the average of \emph{independent}
copies of the $\vZ_s$'s.
More precisely, we will apply the tail inequality to independent
copies $\wt{\vZ}_s$, $s\in \iset{\mu_H}$ of the random variables
$\vZ_s$ and then relate the average of $\wt{\vZ}_s$ to that of
$\vZ_s$.
The following probability inequality is from \citet[Theorem
6.1.1.]{tropp2015intro}.
\begin{theorem}[Matrix Bernstein inequality]
\label{thm:mxbernstein}
Let $\vQ_1,\vQ_2,\dotsc,\vQ_m$ be a sequence of independent, random
$d_1 \times d_2$ matrices.
Assume that $\EE{\vQ_i} = \v0$ and $\Norm{\vQ_i}\le R$ for each $1\le
i \le m$.
Let $\vS = \sum_{i=1}^m \vQ_i$ and let 
\begin{align*}
v = \max\left\{ \norm{\bbE \textstyle\sum_i \vQ_i \vQ_i^\top  }, 
						      \norm{\bbE \textstyle\sum_i \vQ_i^\top \vQ_i  }
			\right\}\,.
\end{align*}
Then, for all $t\ge 0$, 
\[
\bbP\left( \Norm{\vS} \ge t \right) \le 2(d_1 + d_2) \exp\left( -\frac{t^2/2}{v+R t/3}\right)\,.
\]
In other words, for any $\delta \in (0,1)$,
\[
  \bbP\Parens{
    \norm{\vS} > \sqrt{2 v \ln \frac{2(d_1+d_2)}{\delta}} + \frac{2R}{3}
    \ln\frac{2(d_1+d_2)}{\delta}
  } \leq \delta
  \,.
\]
\end{theorem}

To apply \cref{thm:mxbernstein}, it suffices to bound the spectral
norms of $\vZ_s$ (almost surely), $\bbE(\vZ_s \vZ_s^\t)$, and
$\bbE(\vZ_s^\t \vZ_s)$.

\paragraph{Range bound.}
By the triangle inequality,
\[
  \norm{\vZ_s}
  \leq \norm{\Diag(\vpi)^{-1/2} \vY_s \Diag(\vpi)^{-1/2}}
  + \norm{\Diag(\vpi)^{-1/2} \bbE(\vY_s) \Diag(\vpi)^{-1/2}}
  \,.
\]
For the first term, we have
\begin{align}
\label{eq:vysbound}
  \norm{\Diag(\vpi)^{-1/2} \vY_s \Diag(\vpi)^{-1/2}} \leq \frac1{\pimin} .
\end{align}
For the second term, we use the fact $\norm{\vL}\le 1$ to bound
\begin{align*}
  \norm{\Diag(\vpi)^{-1/2} (\bbE(\vY_s) -\vM)\Diag(\vpi)^{-1/2}}
  & = \norm{ \bigl(\Diag(\vpi^{(H_s)}) \Diag(\vpi)^{-1} - \vI  \bigr) \vL }
  \\
  & \le \norm{ \Diag(\vpi^{(H_s)}) \Diag(\vpi)^{-1} - \vI }\,.
\end{align*}
Then, using~\cref{eq:dvpi-ratio-bound},
\begin{equation}
  \norm{ \Diag(\vpi^{(H_s)}) \Diag(\vpi)^{-1} - \vI }
  \le \frac{(1-\gap)^{a'+2(s-1)a}}{\pimin}
  \le \frac{(1-\gap)^{a}}{\pimin} \le 1\,,
  \label{eq:ysbias}
\end{equation}
where the last inequality follows from the assumption that the
block length $a$ satisfies~\cref{eq:block-length}.
Combining this with
$\norm{\Diag(\vpi)^{-1/2} \vM\Diag(\vpi)^{-1/2}} = \norm{\vL}\le 1$,
it follows that
\begin{align}
\label{eq:evysbound}
\norm{\Diag(\vpi)^{-1/2} \bbE(\vY_s) \Diag(\vpi)^{-1/2}} \le 2
\end{align}
by the triangle inequality.
Therefore, together with~\cref{eq:vysbound}, we obtain the range bound
\[
  \norm{\vZ_s} \leq \frac1{\pimin} + 2
  .
\]

\paragraph{Variance bound.}
We now determine bounds on the spectral norms of $\bbE(\vZ_s
\vZ_s^\t)$ and $\bbE(\vZ_s^\t \vZ_s)$.
Observe that
\begin{align}
  \lefteqn{
    \bbE(\vZ_s\vZ_s^\t)
  } \notag \\
  & =
  \frac1{a^2} \sum_{t \in H_s}
  \bbE\Parens{
    \Diag(\vpi)^{-1/2}
    \ve_{X_t} \ve_{X_{t+1}}^\t 
    \Diag(\vpi)^{-1}
    \ve_{X_{t+1}} \ve_{X_t}^\t
    \Diag(\vpi)^{-1/2}
  }
  \label{eq:var1}
  \\
  & \quad
  + \frac1{a^2} \sum_{ \substack{ t \neq t' \\ t,t'\in H_s} } 
  \bbE\Parens{
    \Diag(\vpi)^{-1/2}
    \ve_{X_t} \ve_{X_{t+1}}^\t 
    \Diag(\vpi)^{-1}
    \ve_{X_{t'+1}} \ve_{X_{t'}}^\t
    \Diag(\vpi)^{-1/2}
  }
  \label{eq:var2}
  \\
  & \quad
  - \Diag(\vpi)^{-1/2}
  \bbE(\vY_s)
  \Diag(\vpi)^{-1} 
  \bbE(\vY_s^\t)
  \Diag(\vpi)^{-1/2}
  .
  \label{eq:var3}
\end{align}
The first sum, \cref{eq:var1}, easily simplifies to the diagonal matrix
\begin{align*}
  \lefteqn{
    \frac1{a^2}
    \sum_{t \in H_s}
    \sum_{i=1}^d \sum_{j=1}^d \Pr(X_t = i, X_{t+1} = j) \cdot
    \frac{1}{\pi_i\pi_j} \ve_i\ve_j^\t\ve_j\ve_i^\t
  } \\
  & =
  \frac1{a^2}
  \sum_{t \in H_s}
  \sum_{i=1}^d \sum_{j=1}^d \pi_i^{(t)} P_{i,j} \cdot
  \frac{1}{\pi_i\pi_j} \ve_i\ve_i^\t
  =
  \frac1{a}
  \sum_{i=1}^d \frac{\pi_i^{(H_s)}}{\pi_i}
  \Parens{ \sum_{j=1}^d \frac{P_{i,j}}{\pi_j} }
   \ve_i \ve_i^\t
  .
\end{align*}
For the second sum, \cref{eq:var2},  a symmetric matrix, consider
\[
  \vu^\t
  \Parens{
    \frac1{a^2} \sum_{ \substack{ t \neq t' \\ t,t'\in H_s} } 
    \bbE\Parens{
      \Diag(\vpi)^{-1/2}
      \ve_{X_t} \ve_{X_{t+1}}^\t 
      \Diag(\vpi)^{-1}
      \ve_{X_{t'+1}} \ve_{X_{t'}}^\t
      \Diag(\vpi)^{-1/2}
    }
  } \vu
\]
for an arbitrary unit vector $\vu$.
By Cauchy-Schwarz and AM/GM, this is bounded from above by
\begin{multline*}
  \frac1{2a^2}
  \sum_{ \substack{ t \neq t' \\ t,t'\in H_s} }
  \biggl[
  \bbE\Parens{
    \vu^\t
    \Diag(\vpi)^{-1/2} \ve_{X_t} \ve_{X_{t+1}}^\t \Diag(\vpi)^{-1}
    \ve_{X_{t+1}} \ve_{X_t}^\t \Diag(\vpi)^{-1/2}
    \vu
  }
  \\
  +
  \bbE\Parens{
    \vu^\t
    \Diag(\vpi)^{-1/2} \ve_{X_{t'}} \ve_{X_{t'+1}}^\t \Diag(\vpi)^{-1}
    \ve_{X_{t'+1}} \ve_{X_{t'}}^\t \Diag(\vpi)^{-1/2}
    \vu
  }
  \biggr]
  ,
\end{multline*}
which simplifies to
\[
  \frac{a-1}{a^2}
  \vu^\t \bbE\Parens{
    \sum_{t \in H_s}
    \Diag(\vpi)^{-1/2} \ve_{X_t} \ve_{X_{t+1}}^\t \Diag(\vpi)^{-1}
    \ve_{X_{t+1}} \ve_{X_t}^\t \Diag(\vpi)^{-1/2}
  } \vu\,.
\]
The expectation is the same as that for the first term,
\cref{eq:var1}.

Finally, the spectral norm of the third term, \cref{eq:var3}, is
bounded using \cref{eq:evysbound}:
\[
  \norm{\Diag(\vpi)^{-1/2} \bbE(\vY_s) \Diag(\vpi)^{-1/2}}^2
  \leq 4
  .
\]

Therefore, by the triangle inequality, the bound $\pi_i^{(H)}/\pi_i
\leq 2$ from \cref{eq:ysbias}, and simplifications, 
\begin{align*}
  \Norm{
    \bbE(\vZ_s \vZ_s^\t)
  }
  &\leq
  \max_{i\in[d]}
  \Parens{ \sum_{j=1}^d \frac{P_{i,j}}{\pi_j} }
  \frac{\pi_i^{(H)}}{\pi_i}
  + 4
  \leq
  2\max_{i\in[d]}
  \Parens{ \sum_{j=1}^d \frac{P_{i,j}}{\pi_j} }
  + 4
  .
\end{align*}

We can bound $\bbE(\vZ_s^\t \vZ_s)$
in a similar way; the only difference is that the reversibility needs
to be used at one place to simplify an expectation:
\begin{align*}
\MoveEqLeft
    \frac1{a^2} \sum_{t \in H_s}
    \bbE\Parens{
      \Diag(\vpi)^{-1/2}
      \ve_{X_{t+1}} \ve_{X_t}^\t
      \Diag(\vpi)^{-1}
      \ve_{X_t}\ve_{X_{t+1}}^\t
      \Diag(\vpi)^{-1/2}
    }
	\\
  & =
  \frac1{a^2} \sum_{t \in H_s}
  \sum_{i=1}^d \sum_{j=1}^d \Pr(X_t = i, X_{t+1} = j) \cdot
  \frac{1}{\pi_i\pi_j} \ve_j\ve_j^\t
  \\
  & =
  \frac1{a^2} \sum_{t \in H_s}
  \sum_{i=1}^d \sum_{j=1}^d \pi_i^{(t)} P_{i,j} \cdot
  \frac{1}{\pi_i\pi_j} \ve_j\ve_j^\t
  \\
  & =
  \frac1{a^2} \sum_{t \in H_s}
  \sum_{j=1}^d \Parens{
    \sum_{i=1}^d \frac{\pi_i^{(t)}}{\pi_i} \cdot \frac{P_{j,i}}{\pi_i}
  } \ve_j\ve_j^\t
\end{align*}
where the last step uses \cref{eq:reversibility}.
As before, we get
\begin{align*}
  \Norm{
    \bbE(\vZ_s^\t\vZ_s)
  }
  &\leq
  \max_{i\in[d]}
  \Parens{
    \sum_{j=1}^d \frac{P_{i,j}}{\pi_j}
    \cdot \frac{\pi_j^{(H)}}{\pi_j}
  }
  + 4
  \leq
  2 \max_{i\in[d]}
  \Parens{
    \sum_{j=1}^d \frac{P_{i,j}}{\pi_j}
  }
  + 4
\end{align*}
again using the bound $\pi_i^{(H)}/\pi_i \leq 2$ from
\cref{eq:ysbias}.

\paragraph{Independent copies bound.}
Let $\wt\vZ_s$ for $s \in [\mu_H]$ be independent copies of
$\vZ_s$ for $s \in [\mu_H]$.
Applying \cref{thm:mxbernstein} to the average of these random
matrices, we have
\begin{equation}
  \bbP\Parens{
    \Norm{ \frac1{\mu_H} \sum_{s=1}^{\mu_H} \wt\vZ_s}
    >
    \sqrt{
      \frac{4\Parens{ d_{\vP} + 2 } \ln\frac{4d}{\delta}}
      {\mu_H}
    }
    + \frac{2\Parens{ \frac1{\pimin} + 2 } \ln\frac{4d}{\delta}}
    {3\mu_H}
  } \leq \delta
  \label{eq:indpt-bound}
\end{equation}
where
\begin{align*}
  d_{\vP}
  & := \max_{i \in [d]} \sum_{j=1}^d \frac{P_{i,j}}{\pi_j}
  \leq
  \frac{1}{\pimin}
  \,.
\end{align*}

\paragraph{The actual bound.}
To bound the probability that $\norm{\sum_{s=1}^{\mu_H} \vZ_s/\mu_H}$
is large, we appeal to the following result, which is a consequence of
\citep[Corollary 2.7]{Yu94}.
For each $s \in [\mu_H]$, let $X^{(H_s)} := (X_t : a' + 2(s-1)a + 1
\leq t \leq a' + (2s-1)a + 1)$, which are the random variables
determining $\vZ_s$.
Let $\bbP$ denote the joint distribution of $(X^{(H_s)} : s \in
[\mu_H])$; let $\bbP_s$ be its marginal over $X^{(H_s)}$, and let
$\bbP_{1:s+1}$ be its marginal over
$(X^{(H_1)},X^{(H_2)},\dotsc,X^{(H_{s+1})})$.
Let $\wt\bbP$ be the product distribution formed from the marginals
$\bbP_1,\bbP_2,\dotsc,\bbP_{\mu_H}$, so $\wt\bbP$ governs the joint
distribution of $(\wt\vZ_s : s \in [\mu_H])$.
The result from \citep[Corollary 2.7]{Yu94} implies for any event $E$,
\[
  |\bbP(E) - \wt\bbP(E)| \leq (\mu_H-1) \beta(\bbP)
\]
where
\[
  \beta(\bbP)
  :=
  \max_{1 \leq s \leq \mu_H-1}
  \bbE\Parens{
    \Norm{
      \bbP_{1:s+1}(\cdot\,|X^{(H_1)},X^{(H_2)},\dotsc,X^{(H_s)}) - \bbP_{s+1}
    }_{\tv}
  }
  \,.
\]
Here, $\norm{\cdot}_{\tv}$ denotes the total variation norm.
The number $\beta(\bbP)$ can be recognized to be the
\emph{$\beta$-mixing coefficient} of the stochastic process
$\{X^{(H_s)}\}_{s \in [\mu_H]}$.
This result implies that the bound from \cref{eq:indpt-bound} for
$\norm{\sum_{s=1}^{\mu_H} \wt\vZ_s/\mu_H}$ also holds for
$\norm{\sum_{s=1}^{\mu_H} \vZ_s/\mu_H}$, except the probability
bound increases from $\delta$ to $\delta + (\mu_H-1)\beta(\bbP)$:
\begin{equation}
  \bbP\Parens{
    \Norm{ \frac1{\mu_H} \sum_{s=1}^{\mu_H} \vZ_s}
    >
    \sqrt{
      \frac{4\Parens{ d_{\vP} + 2 } \ln\frac{4d}{\delta}}
      {\mu_H}
    }
    + \frac{2\Parens{ \frac1{\pimin} + 2 } \ln\frac{4d}{\delta}}
    {3\mu_H}
  } \leq \delta + (\mu_H-1)\beta(\bbP)
  .
  \label{eq:dpt-bound}
\end{equation}
By the triangle inequality,
\[
  \beta(\bbP)
  \leq
  \max_{1 \leq s \leq \mu_H-1}
  \bbE\Parens{
    \vphantom{\bigg\vert}
    \Norm{
      \bbP_{1:s+1}(\cdot\,|X^{(H_1)},X^{(H_2)},\dotsc,X^{(H_s)})
      - \bbP^{\vpi}
    }_{\tv}
    +
    \Norm{
      \bbP_{s+1}
      - \bbP^{\vpi}
    }_{\tv}
  }
\]
where $\bbP^{\vpi}$ is the marginal distribution of $X^{(H_1)}$ under
the stationary chain.
Using the Markov property and integrating out $X_t$ for $t >
\min H_{s+1} = a'+2sa+1$,
\[
  \Norm{
    \bbP_{1:s+1}(\cdot\,|X^{(H_1)},X^{(H_2)},\dotsc,X^{(H_s)})
    - \bbP^{\vpi}
  }_{\tv}
  =
  \Norm{
    \cL(X_{a'+2sa+1}\,|X_{a'+(2s-1)a+1})
    - \vpi
  }_{\tv}
\]
where $\cL(Y|Z)$ denotes the conditional distribution of $Y$ given
$Z$.
We bound this distance using standard arguments for bounding the
mixing time in terms of the \emph{relaxation time}
$1/\gap$~\citep[see, e.g., the proof of Theorem 12.3 of][]{LePeWi08}:
for any $i \in [d]$,
\[
  \Norm{
    \cL(X_{a'+2sa+1}\,|X_{a'+(2s-1)a+1}=i)
    - \vpi
  }_{\tv}
  =
  \Norm{
    \cL(X_{a+1}\,|X_1=i)
    - \vpi
  }_{\tv}
  \leq
  \frac{
    \exp\Parens{ -a\gap }
  }{\pimin}
  .
\]
The distance $\norm{ \bbP_{s+1} - \bbP^{\vpi} }_{\tv}$ can be bounded
similarly:
\begin{align*}
  \Norm{
    \bbP_{s+1}
    - \bbP^{\vpi}
  }_{\tv}
  & =
  \Norm{
    \cL(X_{a'+2sa+1})
    - \vpi
  }_{\tv}
  \\
  & =
  \Norm{
    \sum_{i=1}^d \bbP(X_1 = i)
    \cL(X_{a'+2sa+1}\,|X_1 = i)
    - \vpi
  }_{\tv}
  \\
  & \leq
  \sum_{i=1}^d \bbP(X_1 = i)
  \Norm{
    \cL(X_{a'+2sa+1}\,|X_1 = i)
    - \vpi
  }_{\tv}
  \\
  & \leq
  \frac{
    \exp\Parens{-(a'+2sa)\gap}
  }{\pimin}
  \leq
  \frac{
    \exp\Parens{-a\gap}
  }{\pimin}
  .
\end{align*}
We conclude
\[
  (\mu_H-1)\beta(\bbP)
  \leq (\mu_H-1)\frac{2\exp(-a\gap)}{\pimin}
  \leq \frac{2(n-2)\exp(-a\gap)}{\pimin}
  \leq \delta
\]
where the last step follows from the block length assumption
\cref{eq:block-length}.

We return to the decomposition from \cref{eq:blocking}.
We apply \cref{eq:dpt-bound} to both the $H_s$ blocks and the $T_s$
blocks, and combine with \cref{eq:first-block} to obtain the following
probabilistic bound.
Pick any $\delta \in (0,1)$, let the block length be
\[
  a
  := \ceil{ a_\delta }
  =
  \Ceil{
    \frac{1}{\gap}\ln\frac{2(n-2)}{\pimin\delta}
  }
  ,
\]
so
\[
  \min\{\mu_H,\mu_T\}
  =
  \Floor{
    \frac{n-1-a'}{2a}
  }
  \geq
  \frac{n-1}
  {
    2\Parens{
      1 + \frac1{\gap}\ln\frac{2(n-2)}{\pimin\delta}
    }
  }
  -2
  =: \mu
  .
\]
If
\begin{equation}
  n \geq 7 + \frac6{\gap} \ln \frac{2(n-2)}{\pimin\delta}
  \geq 3a
  ,
  \label{eq:min-n2}
\end{equation}
then with probability at least $1-4\delta$,
\begin{multline*}
  \Norm{\Diag(\vpi)^{-1/2}\Parens{\wh\vM - \bbE[\wh\vM]} \Diag(\vpi)^{-1/2} }
  \\
  \leq
  \frac{4
    \Ceil{
      \frac{1}{\gap}\ln\frac{2(n-2)}{\pimin\delta}
    }
  }{\pimin(n-1)}
  +
  \sqrt{
    \frac{4\Parens{ d_{\vP} + 2 } \ln\frac{4d}{\delta}}
    {\mu}
  }
  + \frac{2\Parens{ \frac1{\pimin} + 2 } \ln\frac{4d}{\delta}}
  {3\mu}
  .
\end{multline*}

\subsubsection{The bound on $\norm{\errm}$}
Combining the probabilistic bound from above with the bound on the
bias from \cref{eq:biasbound}, we obtain the following.
Assuming the condition on $n$ from \cref{eq:min-n2}, with probability
at least $1-4\delta$,
\begin{equation}
  \norm{\errm}
  \leq
  \frac{1}{(n-1)\gap\pimin}
  +
  \frac{4
    \Ceil{
      \frac{1}{\gap}\ln\frac{2(n-2)}{\pimin\delta}
    }
  }{\pimin(n-1)}
  +
  \sqrt{
    \frac{4\Parens{ d_{\vP} + 2 } \ln\frac{4d}{\delta}}
    {\mu}
  }
  + \frac{2\Parens{ \frac1{\pimin} + 2 } \ln\frac{4d}{\delta}}
  {3\mu}
  .
  \label{eq:doubletbound}
\end{equation}

\subsection{Overall error bound}
Assume the sequence length $n$ satisfies \cref{eq:min-n1} and
\cref{eq:min-n2}.
Consider the $1-5\delta$ probability event in which
\cref{eq:singletonboundsimple,eq:hatpimin,eq:doubletbound} hold.
In this event, \cref{eq:singletonbound} also holds, and hence by
\cref{eq:vlsimple},
\begin{multline*}
  \norm{\wh\vL - \vL}
  \leq
  3
  \Parens{
    \sqrt{
      \frac{
        8\ln\Parens{\frac{d}\delta\sqrt{\frac{2}{\pimin}}}
      }{
        \pimin\gap n
      }
    }
    +
    \frac{
      20\ln\Parens{\frac{d}\delta\sqrt{\frac{2}{\pimin}}}
    }{
      \pimin\gap n
    }
  }
  \\
  +
  3\Parens{
    \frac{1}{(n-1)\gap\pimin}
    +
    \frac{4
      \Ceil{
        \frac{1}{\gap}\ln\frac{2(n-2)}{\pimin\delta}
      }
    }{\pimin(n-1)}
    +
    \sqrt{
      \frac{4\Parens{ d_{\vP} + 2 } \ln\frac{4d}{\delta}}
      {\mu}
    }
    + \frac{2\Parens{ \frac1{\pimin} + 2 } \ln\frac{4d}{\delta}}
    {3\mu}
  }
  \\
  =
  O
  \Parens{
    \sqrt{
      \frac{
        \log\Parens{\frac{d}{\delta}}
        \log\Parens{\frac{n}{\pimin\delta}}
      }{\pimin\gap n}
    }
  }
  .
\end{multline*}

To finish the proof of \cref{lem:gap}, we replace $\delta$ with
$\delta/5$, and now observe that the bound on $\norm{\wh\vL-\vL}$ is
trivial if \cref{eq:min-n1} is violated (recalling that
$\norm{\wh\vL-\vL} \leq 2$ always holds).
We tack on an additional term $\log(1/(\pimin\gap\delta)) / (\gap n)$
to also ensure a trivial bound if \cref{eq:min-n2} is violated.
\hfill $\qed$

\section{Proof of \cref{thm:empirical}}\label{app:empirical}
In this section, we derive \cref{alg:empest} and prove
\cref{thm:empirical}.

\subsection{Estimators for $\vpi$ and $\gap$}

The algorithm forms the estimator $\wh\vP$ of $\vP$ using Laplace smoothing:
\[
  \wh{P}_{i,j}
  := \frac{N_{i,j} + \alpha}{N_i + d\alpha}
\]
where
\[
  N_{i,j} := \Abs{ \Braces{ t \in [n-1] : (X_t,X_{t+1}) = (i,j) } } ,
  \quad
  N_i := \Abs{ \Braces{ t \in [n-1] : X_t = i } }
\]
and $\alpha>0$ is a positive constant, which we set beforehand as $\alpha := 1/d$
for simplicity.

As a result of the smoothing, all entries of $\wh\vP$ are positive,
and hence $\wh\vP$ is a transition probability
matrix for an ergodic Markov chain.
We let $\hat\vpi$ be the unique stationary distribution for $\wh\vP$.
Using $\hat\vpi$, we form an estimator $\Sym(\wh\vL)$ of $\vL$ using:
\[
  \Sym(\wh\vL) := \frac12 \parens{ \wh\vL + \wh\vL^\t }
  ,
  \qquad
  \wh\vL := \Diag(\hat\vpi)^{1/2} \wh\vP \Diag(\hat\vpi)^{-1/2}
  .
\]
Let $\hat\lambda_1 \geq \hat\lambda_2 \geq \dotsb \geq \hat\lambda_d$
be the eigenvalues of $\Sym(\wh\vL)$ (and in fact, we have $1 =
\hat\lambda_1 > \hat\lambda_2$ and $\hat\lambda_d > -1$).
The algorithm estimates the spectral gap $\gap$ using
\[
  \hatgap := 1 - \max\{ \hat\lambda_2, |\hat\lambda_d| \} .
\]

\subsection{Empirical bounds for $\vP$}
\label{sec:P-obs-bound}

We make use of a simple corollary of Freedman's inequality for
martingales~\citep[Theorem 1.6]{Fre75}.
\begin{theorem}[Freedman's inequality]
  \label{thm:freedman}
  Let $(Y_t)_{t \in \bbN}$ be a bounded martingale difference sequence
  with respect to the filtration $\cF_0 \subset \cF_1 \subset \cF_2
  \subset \dotsb$; assume for some $b>0$, $|Y_t| \leq b$ almost surely
  for all $t \in \bbN$.
  Let $V_k := \sum_{t=1}^k \EE{Y_t^2|\cF_{t-1}}$ and $S_k :=
  \sum_{t=1}^k Y_t$ for $k \in \bbN$.
  For all $s, v > 0$,
  \[
    \Pr\Brackets{
      \exists k \in \bbN \;\st\,
      S_k > s
      \,\wedge\,
      V_k \leq v
    }
    \leq \Parens{
      \frac{v/b^2}{s/b+v/b^2}
    }^{s/b+v/b^2}
    e^{s/b}
    = \exp\Parens{
      -\frac{v}{b^2} \cdot h\Parens{\frac{bs}{v}}
    }
    \,,
  \]
  where $h(u) := (1+u)\ln(1+u) - u$.
\end{theorem}
Observe that in \Cref{thm:freedman}, for any $x>0$, if $s :=
\sqrt{2vx} + bx/3$ and $z := b^2x/v$, then the probability bound
on the right-hand side becomes
\[
  \exp\Parens{
    -x \cdot \frac{h\Parens{\sqrt{2z}+z/3}}{z}
  }
  \leq
  e^{-x}
\]
since $h(\sqrt{2z}+z/3)/z \geq 1$ for all $z > 0$ (see,
e.g.,~\citep[proof of Lemma 5]{audibert2009}).

\begin{corollary}
  \label{cor:freedman}
  Under the same setting as \Cref{thm:freedman}, for any $n \geq 1$,
  $x > 0$, and $c > 1$,
  \[
    \Pr\Brackets{
      \exists k \in [n] \;\st\,
      S_k > \sqrt{2cV_kx} + 4bx/3
    }
    \leq
    \Parens{ 1 + \ceil{\log_c(2n/x)}_+ }
    e^{-x}
    .
  \]
\end{corollary}
\begin{proof}
  Define $v_i := c^i b^2x/2$ for $i = 0, 1, 2, \dotsc, \ceil{
  \log_c(2n/x) }_+$, and let $v_{-1} := -\infty$.
  Then, since $V_k \in [0,b^2n]$ for all $k\in[n]$,
  \begin{align*}
    \MoveEqLeft{
      \Pr\Brackets{
        \exists k \in [n] \;\st\,
        S_k > \sqrt{2\max\braces{v_0,cV_k}x} + bx/3
      }
    }
    \\
    & =
    \sum_{i=0}^{\ceil{\log_c(2n/x)}_+}
    \Pr\Brackets{
      \exists k \in [n] \;\st\,
      S_k > \sqrt{2\max\braces{v_0,cV_k}x} + bx/3
      \,\wedge\, v_{i-1} < V_k \leq v_i
    }
    \\
    & \leq
    \sum_{i=0}^{\ceil{\log_c(2n/x)}_+}
    \Pr\Brackets{
      \exists k \in [n] \;\st\,
      S_k > \sqrt{2\max\braces{v_0,cv_{i-1}}x} + bx/3
      \,\wedge\, v_{i-1} < V_k \leq v_i
    }
    \\
    & \leq
    \sum_{i=0}^{\ceil{\log_c(2n/x)}_+}
    \Pr\Brackets{
      \exists k \in [n] \;\st\,
      S_k > \sqrt{2v_ix} + bx/3
      \,\wedge\, V_k \leq v_i
    }
    \\
    & \leq
    \Parens{ 1 + \ceil{\log_c(2n/x)}_+ }
    e^{-x}\,,
  \end{align*}
  where the final inequality uses \Cref{thm:freedman}.
  The conclusion now follows because
  \[
    \sqrt{2cV_kx} + 4bx/3
    \geq \sqrt{2\max\braces{v_0,cV_k}x} + bx/3
  \]
  for all $k \in [n]$.
\end{proof}

\begin{lemma}
  \label{lem:P-unobs-bound}
  The following holds for any constant $c>1$ with probability at least
  $1-\delta$: for all $(i,j) \in \iset{d}^2$,
  \begin{equation}
    \abs{ \wh{P}_{i,j} - P_{i,j} }
    \leq
    \sqrt{\Parens{\frac{N_i}{N_i+d\alpha}}\frac{2cP_{i,j}(1-P_{i,j})\emptail}{N_i+d\alpha}}
    + \frac{(4/3)\emptail}{N_i+d\alpha}
    + \frac{\abs{\alpha-d\alpha P_{i,j}}}{N_i+d\alpha}\,,
    \label{eq:P-unobs-bound}
  \end{equation}
  where
  \begin{equation}
    \emptail
    := \inf
    \Braces{
      t\geq0 :
      2d^2 \Parens{ 1 + \ceil{\log_c(2n/t)}_+ } e^{-t} \leq \delta
    }
    = O\Parens{ \log\Parens{ \frac{d\log(n)}{\delta} } }
    \,.
    \label{eq:emptail}
  \end{equation}
\end{lemma}
\begin{proof}
  Let $\cF_t$ be the $\sigma$-field generated by $X_1,X_2,\dotsc,X_t$.
  Fix a pair $(i,j) \in [d]^2$.
  Let $Y_1 := 0$, and for $t \geq 2$,
  \[
    Y_t := \Ind{X_{t-1} = i} (\Ind{X_t=j} - P_{i,j})
    ,
  \]
  so that
  \[
    \sum_{t=1}^n Y_t
    = N_{i,j} - N_i P_{i,j}
    .
  \]
  The Markov property implies that the stochastic process
  $(Y_t)_{t\in[n]}$ is an $(\cF_t)$-adapted martingale difference
  sequence: $Y_t$ is $\cF_t$-measurable and $\EE{Y_t| \cF_{t-1} } =
  0$, for each $t$.
  Moreover, for all $t \in [n]$,
  \[
    Y_t \in [-P_{i,j},1-P_{i,j}]
    \,,
  \]
  and for $t \geq 2$,
  \[
    \EE{Y_t^2| \cF_{t-1} } = \Ind{X_{t-1}=i} P_{i,j} (1-P_{i,j})
    \,.
  \]
  Therefore, by \Cref{cor:freedman} and union bounds, we have
  \[
    \abs{ N_{i,j} - N_i P_{i,j} }
    \leq \sqrt{2c N_i P_{i,j} (1-P_{i,j}) \emptail} +
    \frac{4\emptail}3
  \]
  for all $(i,j) \in [d]^2$.
\end{proof}

\Cref{eq:P-unobs-bound} can be viewed as constraints on the possible
value that $P_{i,j}$ may have (with high probability).
Since $P_{i,j}$ is the only unobserved quantity in the bound from
\cref{eq:P-unobs-bound}, we can numerically maximize $|\wh{P}_{i,j} -
P_{i,j}|$ subject to the constraint in \cref{eq:P-unobs-bound}
(viewing $P_{i,j}$ as the optimization variable).
Let $B_{i,j}^*$ be this maximum value, so we have
\[
  P_{i,j} \in
  \Brackets{
    \wh{P}_{i,j} - B_{i,j}^*,\,
    \wh{P}_{i,j} + B_{i,j}^*
  }
\]
in the same event where \cref{eq:P-unobs-bound} holds.

In the algorithm, we give a simple alternative to computing
$B_{i,j}^*$ that avoids numerical optimization, derived in the spirit
of empirical Bernstein bounds~\citep{audibert2009}.
Specifically, with $c := 1.1$ (an arbitrary choice), we compute
\begin{equation}
  \wh{B}_{i,j}
  :=
  \Parens{
    \sqrt{\frac{c\emptail}{2N_i}}
    + \sqrt{
      \frac{c\emptail}{2N_i}
      +
      \sqrt{\frac{2c\wh{P}_{i,j}(1-\wh{P}_{i,j})\emptail}{N_i}}
      + \frac{(4/3)\emptail + \abs{\alpha-d\alpha\wh{P}_{i,j}}}{N_i}
    }
  }^2
  \label{eq:P-obs-bound}
\end{equation}
for each $(i,j) \in [d]^2$, where $\emptail$ is defined in
\cref{eq:emptail}.
We show in \cref{lem:P-obs-bound} that
\[
  P_{i,j} \in
  \Brackets{
    \wh{P}_{i,j} - \wh{B}_{i,j},\,
    \wh{P}_{i,j} + \wh{B}_{i,j}
  }
\]
again, in the same event where \cref{eq:P-unobs-bound} holds.
The observable bound in \cref{eq:P-obs-bound} is not too far from the
unobservable bound in~\cref{eq:P-unobs-bound}.

\begin{lemma}
  \label{lem:P-obs-bound}
  In the same $1-\delta$ event as from \cref{lem:P-unobs-bound},
  we have
  $P_{i,j} \in \brackets{ \wh{P}_{i,j} - \wh{B}_{i,j},\, \wh{P}_{i,j}
  + \wh{B}_{i,j} }$
  for all $(i,j) \in [d]^2$,
  where $\wh{B}_{i,j}$ is defined in \cref{eq:P-obs-bound}.
\end{lemma}
\begin{proof}
  Recall that in the $1-\delta$ probability event from
  \cref{lem:P-unobs-bound}, we have for all $(i,j) \in [d]^2$,
  \begin{multline*}
    \abs{ \wh{P}_{i,j} - P_{i,j} }
    =
    \Abs{
      \frac{N_{i,j} - N_i P_{i,j}}{N_i + d\alpha}
      + \frac{\alpha - d\alpha P_{i,j}}{N_i + d\alpha}
    } 
    \leq
    \sqrt{\frac{2cN_iP_{i,j}(1-P_{i,j})\emptail}{(N_i+d\alpha)^2}}
    + \frac{(4/3)\emptail}{N_i+d\alpha}
    + \frac{\abs{\alpha-d\alpha P_{i,j}}}{N_i+d\alpha}
    .
  \end{multline*}
  Applying the triangle inequality to the right-hand side, we obtain
  \begin{align}
    \abs{ \wh{P}_{i,j} - P_{i,j} }
    & \leq
    \sqrt{
      \frac{
        2cN_i\parens{
          \wh{P}_{i,j}(1-\wh{P}_{i,j})
          + \abs{\wh{P}_{i,j} - P_{i,j}}
        } \emptail
      }{(N_i+d\alpha)^2}
    }
    + \frac{(4/3)\emptail}{N_i+d\alpha}
    \notag \\
    & \qquad
    + \frac{
      \abs{\alpha-d\alpha \wh{P}_{i,j}}
      + d\alpha\abs{\wh{P}_{i,j}-P_{i,j}}
    }{N_i+d\alpha}
    .
    \notag
  \end{align}
  Since $\sqrt{A+B} \leq \sqrt{A} + \sqrt{B}$ for non-negative $A,B$, we
  loosen the above inequality and rearrange it to obtain
  \begin{align*}
    \Parens{ 1 - \frac{d\alpha}{N_i+d\alpha} }
    \abs{\wh{P}_{i,j} - P_{i,j}}
    & \leq
    \sqrt{\abs{\wh{P}_{i,j} - P_{i,j}}} \cdot
    \sqrt{\frac{2cN_i\emptail}{(N_i+d\alpha)^2}}
    \\
    & \qquad
    +
    \sqrt{\frac{2cN_i\wh{P}_{i,j}(1-\wh{P}_{i,j})\emptail}{(N_i+d\alpha)^2}}
    + \frac{(4/3)\emptail + \abs{\alpha -
    d\alpha\wh{P}_{i,j}}}{N_i+d\alpha}
    .
  \end{align*}
  Whenever $N_i > 0$, we can solve a quadratic inequality to conclude
  $\abs{\wh{P}_{i,j}-P_{i,j}} \leq \wh{B}_{i,j}$.
\end{proof}

\subsection{Empirical bounds for $\vpi$}

Recall that $\hat\vpi$ is obtained as the unique stationary
distribution for $\wh\vP$.
Let $\wh\vA := \vI - \wh\vP$, and let $\giAh$ be the \emph{group
inverse} of $\wh\vA$---i.e., the unique square matrix satisfying the
following equalities:
\[
  \wh\vA\giAh\wh\vA = \wh\vA , \quad
  \giAh\wh\vA\giAh = \giAh , \quad
  \giAh\wh\vA = \wh\vA\giAh .
\]
The matrix $\giAh$, which is well defined no matter what transition probability matrix $\wh\vP$ we start with
\citep{meyer1975role},
 is a central quantity that captures many properties
of the ergodic Markov chain with transition matrix
$\wh\vP$~\citep{meyer1975role}.
We denote the $(i,j)$-th entry of $\giAh$ by $\giAh_{i,j}$.
Define
\begin{equation}
  \hat\kappa :=
  \frac12
  \max
  \Braces{
    \giAh_{j,j}
    - \min\Braces{ \giAh_{i,j} : i \in [d] }
    : j \in [d]
  }
  .
  \notag
\end{equation}
Analogously define
\begin{align*}
  \vA & := \vI - \vP , \\
  \giA & := \text{group inverse of $\vA$} , \\
  \kappa & :=
  \frac12
  \max
  \Braces{
    \giA_{j,j}
    - \min\Braces{ \giA_{i,j} : i \in [d] }
    : j \in [d]
  }
  .
\end{align*}
We now use the following perturbation bound from \citep[Section
3.3]{cho2001comparison} (derived
from~\citet{haviv1984perturbation,kirkland1998applications}).
\begin{lemma}[\citep{haviv1984perturbation,kirkland1998applications}]
  \label{lem:pi-perturb}
  If
  $\abs{\wh{P}_{i,j} - P_{i,j}} \leq \wh{B}_{i,j}$ for each $(i,j) \in
  [d]^2$, then
  \begin{align}
    \max\Braces{
      \abs{\hat\pi_i - \pi_i}
      : i \in [d]
    }
    & \leq \min\{\kappa,\hat\kappa\} \max
    \braces{
      \wh{B}_{i,j}
      : (i,j) \in [d]^2
    }
    \notag \\
    & \leq \hat\kappa \max
    \braces{
      \wh{B}_{i,j}
      : (i,j) \in [d]^2
    }
    .
    \notag
  \end{align}
\end{lemma}
This establishes the validity of the confidence intervals for the
$\pi_i$ in the same event from \cref{lem:P-unobs-bound}.

We now establish the validity of the bounds for the ratio quantities
$\sqrt{\hat\pi_i/\pi_i}$ and $\sqrt{\pi_i/\hat\pi_i}$.
\begin{lemma}
  \label{lem:pi-ratio-perturb}
  If $\max\braces{\abs{\hat\pi_i-\pi_i} : i \in [d]} \leq \hat{b}$,
  then
  \begin{equation}
    \max\bigcup_{i\in[d]}
    \braces{
      \abs{\sqrt{\pi_i/\hat\pi_i}-1},\,
      \abs{\sqrt{\hat\pi_i/\pi_i}-1}
    }
    \leq
    \frac12 \max \bigcup_{i\in[d]}
    \Braces{
      \frac{\hat{b}}{\hat\pi_i},\,
      \frac{\hat{b}}{\brackets{\hat\pi_i-\hat{b}}_+}
    }
    .
    \notag
  \end{equation}
\end{lemma}
\begin{proof}
  By \cref{lem:pi-perturb}, we have for each $i \in [d]$,
  \begin{equation}
    \frac{\abs{\hat\pi_i - \pi_i}}{\hat\pi_i}
    \leq \frac{\hat{b}}{\hat\pi_i}
    , \quad
    \frac{\abs{\hat\pi_i - \pi_i}}{\pi_i}
    \leq
    \frac{\hat{b}}{\pi_i}
    \leq
    \frac{\hat{b}}{\brackets{\hat{\pi}_i-\hat{b}}_+}
    .
    \notag
  \end{equation}
  Therefore, using the fact that for any $x>0$,
  \[
    \max\Braces{
      |\sqrt{x}-1| ,\, |\sqrt{1/x}-1|
    } \leq \frac12 \max\Braces{ |x-1| ,\, |1/x-1| }
  \]
  we have for every $i \in [d]$,
  \begin{align*}
    \max\Braces{
      \abs{\sqrt{\pi_i/\hat\pi_i}-1}
      ,\,
      \abs{\sqrt{\hat\pi_i/\pi_i}-1}
    }
    & \leq
    \frac12 \max\Braces{
      \abs{\pi_i/\hat\pi_i-1}
      ,\,
      \abs{\hat\pi_i/\pi_i-1}
    }
    \\
    & \leq
    \frac12 \max\Braces{
      \frac{\hat{b}}{\hat\pi_i}
      ,\,
      \frac{\hat{b}}{\brackets{\hat{\pi}_i-\hat{b}}_+}
    }
    .
    \qedhere
  \end{align*}
\end{proof}

\subsection{Empirical bounds for $\vL$}

By Weyl's inequality and the triangle inequality,
\[
  \max_{i \in [d]} |\lambda_i - \hat\lambda_i|
  \leq \norm{\vL - \Sym(\wh\vL)}
  \leq \norm{\vL - \wh\vL} .
\]
It is easy to show that $|\hatgap - \gap|$ is bounded by the same
quantity.
Therefore, it remains to establish an empirical bound on $\norm{\vL -
\wh\vL}$.

\begin{lemma}
  \label{lem:L-obs-bound}
  If
  $\abs{\wh{P}_{i,j} - P_{i,j}} \leq \wh{B}_{i,j}$ for each $(i,j) \in
  [d]^2$ and
  $\max\braces{\abs{\hat\pi_i-\pi_i} : i \in [d]} \leq \hat{b}$,
  then
  \[
    \norm{\wh\vL - \vL}
    \leq
    2\hat\rho + \hat\rho^2
    + (1+2\hat\rho+\hat\rho^2)
    \Biggl(
      \sum_{(i,j)\in[d]^2} \frac{\hat\pi_i}{\hat\pi_j} \hat{B}_{i,j}^2
    \Biggr)^{1/2}
    ,
  \]
  where
  \[
    \hat\rho := \frac12 \max \bigcup_{i\in[d]}
    \Braces{
      \frac{\hat{b}}{\hat\pi_i},\,
      \frac{\hat{b}}{\brackets{\hat\pi_i-\hat{b}}_+}
    }
    .
  \]
\end{lemma}
\begin{proof}
  We use the following decomposition of $\vL - \wh\vL$:
  \[
    \vL - \wh\vL
    = \errP
    + \errpil \wh\vL
    + \wh\vL \errpir
    + \errpil \errP
    + \errP \errpir
    + \errpil \wh\vL \errpir
    + \errpil \errP \errpir
  \]
  where
  \begin{align*}
    \errP
    & := \Diag(\hat\vpi)^{1/2} \parens{ \vP - \wh\vP } \Diag(\hat\vpi)^{-1/2} ,
    \\
    \errpil
    & := \Diag(\vpi)^{1/2} \Diag(\hat\vpi)^{-1/2} - \vI ,
    \\
    \errpir
    & := \Diag(\hat\vpi)^{1/2} \Diag(\vpi)^{-1/2} - \vI .
  \end{align*}
  Therefore
  \begin{align*}
    \norm{\vL - \wh\vL}
    & \leq
    \norm{\errpil} + \norm{\errpir} + \norm{\errpil} \norm{\errpir}
    \\
    & \quad
    + \Parens{
      1 + \norm{\errpil} + \norm{\errpir} + \norm{\errpil} \norm{\errpir}
    }
    \norm{\errP}
    .
  \end{align*}
  Observe that for each $(i,j) \in [d]^2$, the $(i,j$)-th entry of
  $\errP$ is bounded in absolute value by
  \[
    |(\errP)_{i,j}|
    = \hat{\pi}_i^{1/2} \hat{\pi}_j^{-1/2} |P_{i,j} - \wh{P}_{i,j}|
    \leq
    \hat\pi_i^{1/2} \hat{\pi}_j^{-1/2} \wh{B}_{i,j}
    .
  \]
  Since the spectral norm of $\errP$ is bounded above by its Frobenius
  norm,
  \[
    \norm{\errP}
    \leq \Biggl(
      \sum_{(i,j) \in [d]^2}
      (\errP)_{i,j}^2
    \Biggr)^{1/2}
    \leq \Biggl(
      \sum_{(i,j) \in [d]^2}
      \frac{\pi_i}{\pi_j} \wh{B}_{i,j}^2
    \Biggr)^{1/2}
    .
  \]
  Finally, the spectral norms of $\errpil$ and $\errpir$ satisfy
  \[
    \max\Braces{ \norm{\errpil} ,\, \norm{\errpir} }
    =
    \max\bigcup_{i\in[d]}
    \braces{
      \abs{\sqrt{\pi_i/\hat\pi_i}-1},\,
      \abs{\sqrt{\hat\pi_i/\pi_i}-1}
    }
    ,
  \]
  which can be bounded using \cref{lem:pi-ratio-perturb}.
\end{proof}

This establishes the validity of the confidence interval for $\gap$ in
the same event from \cref{lem:P-unobs-bound}.

\subsection{Asymptotic widths of intervals}
\label{sec:asymptotic}

Let us now turn to the asymptotic behavior of the interval widths
(regarding $\hat{b}$, $\hat\rho$, and $\hat{w}$ all as
functions of $n$).

A simple calculation gives that, almost surely, as $n\to \infty$,
\begin{align*}
  \sqrt{\frac{n}{\log\log n}}
  \hat{b}
  & =
  O\Parens{
    \max_{i,j} \kappa \sqrt{\frac{P_{i,j}}{\pi_i}}
  }
  ,
  \\
  \sqrt{\frac{n}{\log\log n}}
  \hat\rho
  & =
  O\Parens{
    \frac{\kappa}{\pimin^{3/2}}
  }
  .
\end{align*}
Here, we use the fact that $\hat\kappa \to \kappa$ as $n\to\infty$
since $\giAh \to \giA$ as $\wh\vP \to
\vP$~\citep{li2001improvement,benitez2012continuity}.

Further, since
\begin{align*}
  \sqrt{\frac{n}{\log\log n}}
  \Biggl(
    \sum_{i,j} \frac{\hat{\pi}_i}{\hat{\pi}_j} \wh{B}_{i,j}^2
  \Biggr)^{1/2}
  & =
  O\Parens{
    \Biggl(
      \sum_{i,j} \frac{{\pi}_i}{{\pi}_j}
      \cdot \frac{P_{i,j}(1-P_{i,j})}{\pi_i}
    \Biggr)^{1/2}
  }
  =
  O\Parens{ \sqrt{ \frac{d}{\pimin}} }
  \,,
\end{align*}
we thus have
\begin{align*}
  \sqrt{\frac{n}{\log\log n}}
  \hat{w}
  =
  O\Parens{
    \frac{\kappa}{\pimin^{3/2}} +
    \sqrt{\frac{d}{\pimin}}
  }
  .
\end{align*}
This completes the proof of \cref{thm:empirical}.
\hfill $\qed$

The following claim provides a bound on $\kappa$ in terms of the
number of states and the spectral gap.
\begin{claim}
  \label{claim:kappa-bound}
  $\kappa \leq d/\gap$.
\end{claim}
\begin{proof}
  It is established by Cho and Meyer~\citep{cho2001comparison} that
  \[
    \kappa
    \leq \max_{i,j} |\giA_{i,j}|
    \leq
    \sup_{\norm{\vv}_1 = 1 , \dotp{\vv,\v1} = 0} \norm{\vv^\t \giA}_1
  \]
  (our $\kappa$ is the $\kappa_4$ quantity from
  \citep{cho2001comparison}), and Seneta~\citep{seneta1993sensitivity}
  establishes
  \[
    \sup_{\norm{\vv}_1 = 1 , \dotp{\vv,\v1} = 0} \norm{\vv^\t \giA}_1
    \leq
    \frac{d}{\gap}
    .
    \qedhere
  \]
\end{proof}

\section{Proof of \cref{thm:combined}}\label{app:combined}
\newcommand\LB{\ensuremath{\operatorname{lb}}}
\newcommand\piminlb{\ensuremath{\hat\pi_{\star,\LB}}}
\newcommand\gaplb{\ensuremath{\hat\gamma_{\star,\LB}}}
Let $\piminlb$ and $\gaplb$ be the lower bounds on $\pimin$ and
$\gap$, respectively, computed from \cref{alg:empest}.
Let $\hatpimin$ and $\hatgap$ be the estimates of $\pimin$ and $\gap$
computed using the estimators from \cref{thm:err}.
By a union bound, we have by \cref{thm:err} and \cref{thm:empirical}
that with probability at least $1-2\delta$,
\begin{equation}
  \Abs{\hatpimin-\pimin}
  \le
  C \,
  \Parens{
    \sqrt{\frac{\pimin\log\frac{d}{\piminlb\delta}}{\gaplb n}}
    +
    \frac{\log\frac{d}{\piminlb\delta}}{\gaplb n}
  }
  \label{eq:pimin-plugin}
\end{equation}
and
\begin{equation}
  \Abs{\hatgap-\gap}
  \leq
  C \,
  \Parens{
    \sqrt{\frac{\log\frac{d}{\delta}\cdot\log\frac{n}{\piminlb\delta}}{\piminlb\gaplb n}}
    +
    \frac{\log\frac{d}{\delta}\cdot\log\frac{n}{\piminlb\delta}}{\piminlb\gaplb n}  
    + \frac{\log\frac{1}{\gaplb}}{\gaplb n}  
  }
  \,.
  \label{eq:gap-plugin}
\end{equation}
The bound on $\abs{\hatgap-\gap}$ in \cref{eq:gap-plugin}---call it
$\hat{w}'$---is fully observable and hence yields a confidence
interval for $\gap$.
The bound on $\abs{\hatpimin-\pimin}$ in \cref{eq:pimin-plugin}
depends on $\pimin$, but from it one can derive
\[
  \Abs{\hatpimin-\pimin}
  \le
  C' \,
  \Parens{
    \sqrt{\frac{\hatpimin\log\frac{d}{\piminlb\delta}}{\gaplb n}}
    +
    \frac{\log\frac{d}{\piminlb\delta}}{\gaplb n}
  }
\]
using the approach from the proof of \cref{lem:P-obs-bound}.
Here, $C'>0$ is an absolute constant that depends only on $C$.
This bound---call it $\hat{b}'$---is now also fully observable.
We have established that in the $1-2\delta$ probability event from
above,
\[
  \pimin \in \wh{U} := [\hatpimin-\hat{b}', \hatpimin+\hat{b}']
  , \quad
  \gap \in \wh{V} := [\hatgap-\hat{w}', \hatgap+\hat{w}']
  .
\]
It is easy to see that almost surely (as $n \to \infty$),
\[
  \sqrt{\frac{n}{\log n}} \hat{w}'
  = O\Parens{
    \sqrt{\frac{\log(d/\delta)}{\pimin\gap}}
  }
\]
and
\[
  \sqrt{n} \hat{b}'
  = O\Parens{
    \sqrt{\frac{\pimin \log\frac{d}{\pimin\delta}}{\gap}}
  }
  .
\]
This completes the proof of \cref{thm:combined}.
\hfill\qed

\end{document}